\documentclass{article}

\PassOptionsToPackage{table}{xcolor}
\PassOptionsToPackage{numbers}{natbib}
\usepackage{microtype}
\usepackage{graphicx}
\usepackage{subcaption}
\usepackage{booktabs}
\usepackage{hyperref}

\PassOptionsToPackage{numbers,sort&compress}{natbib}
\usepackage[preprint]{neurips_2026}
\makeatletter\renewcommand{\@noticestring}{}\makeatother

\usepackage{amsmath}
\usepackage{amssymb}
\usepackage{mathtools}
\usepackage{amsthm}
\usepackage{algorithm}
\usepackage{algorithmic}
\usepackage{multirow}
\usepackage{enumitem}
\usepackage{xcolor}
\definecolor{linknavy}{HTML}{1F3D7A}
\definecolor{urlnavy}{HTML}{0B5FA5}
\hypersetup{
  colorlinks=true,
  linkcolor=linknavy,
  citecolor=linknavy,
  urlcolor=urlnavy,
  pdftitle={Fractional is Better: Learnable Derivative Orders in Neural Operator Learning},
  pdfauthor={Fares B. Mehouachi and Saif Eddin Jabari},
  pdfsubject={Neural operator learning; fractional calculus; approximation theory},
  pdfkeywords={neural operators, fractional derivatives, PDE surrogates, operator learning}
}
\usepackage[capitalize,noabbrev]{cleveref}

\theoremstyle{plain}
\newtheorem{theorem}{Theorem}[section]
\newtheorem{proposition}[theorem]{Proposition}
\newtheorem{lemma}[theorem]{Lemma}

\theoremstyle{definition}
\newtheorem{definition}[theorem]{Definition}
\newtheorem{assumption}[theorem]{Assumption}
\newtheorem{example}[theorem]{Example}
\theoremstyle{remark}
\newtheorem{remark}[theorem]{Remark}

\newcommand{\R}{\mathbb{R}}
\newcommand{\E}{\mathbb{E}}

\newcommand{\MSE}{\mathrm{MSE}}
\newcommand{\G}{\mathcal{G}}
\newcommand{\Lop}{\mathcal{L}}
\newcommand{\Non}{\mathcal{N}}
\newcommand{\NO}{\mathcal{NO}}

\newcommand{\norm}[1]{\left\| #1 \right\|}

\newcommand{\Dfrak}{\mathfrak{D}} %

\newenvironment{plainblock}{}{}

\definecolor{impgreen}{HTML}{2E7D32}
\newcommand{\impv}[1]{#1}
\newcommand{\impva}[1]{\textcolor{impgreen}{#1}}

\begin{document}

\title{Fractional is Better: Learnable Derivative Orders in Neural Operator Learning\thanks{Published as a conference paper at ICML 2026.\\
  \hspace*{1.8em}Corresponding author: \texttt{fares.mehouachi@nyu.edu}.}}

\author{%
  Fares B.~Mehouachi$^{1,\ast}$ \\
  \texttt{fares.mehouachi@nyu.edu} \\
  \And
  Saif Eddin Jabari$^{1,2}$ \\
  \texttt{sej7@nyu.edu} \\
  \AND
  \normalfont\vspace{0.6em}
  $^{1}$New York University Abu Dhabi, Abu Dhabi, UAE \\
  $^{2}$New York University, New York, NY, USA \\
}

\maketitle
\setcounter{footnote}{0}   %
\enlargethispage{-2\baselineskip}

\begin{abstract}
Neural operators learn mappings between function spaces, enabling fast surrogate solutions to partial differential equations. Despite remarkable architectural diversity, these methods often share a common input representation: raw coordinate-value pairs $(x, u(x))$. We ask whether inputs aligned with PDE differential structure can improve learning. Through Picard iteration on mild solutions, we show that derivatives of the input appear explicitly in the solution operator, suggesting that providing derivative features should reduce the network's implicit differentiation burden. We prove this intuition: providing derivative features improves approximation bounds from $O(W^{-(s-m)/d})$ to $O(W^{-s/d})$, where $W$ is parameter count, $s$ is input regularity, $m$ is the PDE order, and $d$ is spatial dimension. Our central finding, however, is a surprise: the optimal derivative order $\beta^*$ is strictly less than the PDE order $m$. This gap arises from a bias-variance tradeoff in spectral space that we characterize in closed form. Learning $\beta$ from data achieves automatic spectral regularization. We introduce $\partial$-NO (del-NO), for derivative-augmented neural operators, an architecture-agnostic augmentation that provides learnable fractional derivative features to any neural operator backbone. Across benchmark problems and architectures, $\partial$-NO consistently improves prediction accuracy, with learned orders that reflect a representation of known physics modulated by noise and finite-sample constraints.
\end{abstract}

\section{Introduction}
\label{sec:intro}

Neural operators have emerged as a powerful tool for scientific computing. By learning mappings between function spaces from data, they enable fast surrogate solutions to partial differential equations, with applications spanning climate modeling, materials science, and engineering design \citep{li2020fourier, lu2021learning, kovachki2023neural, yang2025urban, thodi2024fourier}. The field has witnessed rapid architectural innovation: spectral convolutions \citep{li2020fourier}, attention mechanisms \citep{cao2021choose, hao2023gnot}, physics-aware designs \citep{wu2024Transolver}, and latent representations \citep{wu2023solving, mehouachi2026flowmixer}. These advances have substantially improved accuracy and efficiency. Yet beneath this diversity lies a striking uniformity. Virtually all neural operators represent their inputs in the same way:
\begin{equation}
\label{eq:standard_input}
\text{1D: } (x, u(x)), \quad \text{2D: } (x, y, u(x,y)).
\end{equation}
This representation ignores the differential structure that defines the underlying physics. 
Partial differential equations are built from derivatives: $\partial_x u$, $\partial_{xx} u$, and their combinations. A neural operator receiving only function values must learn to differentiate implicitly from finite, noisy samples. This places an unnecessary burden on the network.

\textbf{This paper asks a simple question:} Can we provide inputs that align with PDE differential structure? And if so, what is the best way to do it?

Our answer proceeds in three steps. First, we analyze the structure of PDE solution operators through Picard iteration on mild solutions. This reveals that derivatives of the input function appear explicitly in the solution operator, so providing them as features reduces the implicit differentiation burden on the network (\S\ref{sec:picard}). Second, we prove that providing these features improves approximation bounds by a factor depending on the PDE order (\S\ref{sec:approx}). Third, and most unexpectedly, we uncover an operator-level bias-variance tradeoff: the optimal derivative order $\beta^*$ satisfies
\begin{equation}
\label{eq:main_insight}
\boxed{\beta^* < m}
\end{equation}
for any practically realizable finite, noisy dataset (\S\ref{sec:optimal_beta}). Using the exact PDE order $m$ introduces estimation variance that outweighs the approximation benefit. The gap $m - \beta^*$ shrinks as data grows, but remains positive at every practically realizable scale.

This finding has important practical implications. It means derivative features should not simply match known physics, but should balance physical fidelity against statistical efficiency. Since this optimal order depends on unknown quantities like noise level and sample size, we propose learning it directly from data.

The resulting architecture, which we call $\partial$-NO (del-NO) for derivative-augmented neural operators (Figure~\ref{fig:architecture}; \S\ref{sec:method}), augments any backbone with fractional derivative features whose orders are jointly optimized during training. The method is simple to implement, adds modest overhead, and wraps existing architectures without modification.

Where does $\partial$-NO sit among existing uses of fractional calculus in deep learning? At least three established lines exist: \emph{fractional PINNs} include fractional operators directly in the loss \citep{pang2019fpinns, guo2022monte, ren2023class, javadi2023solving}; \emph{fractional gradient methods} replace integer gradients with fractional ones during optimization \citep{khan2018novel, shin2023accelerating, elnady2025comprehensive}; and \emph{fractional architectures} embed fractional dynamics into the network itself \citep{pu2016fractional, cui2025neural, kang2024unleashing, coelho2024neural}. $\partial$-NO occupies a distinct fourth role: learnable fractional derivative \emph{input features} applied as preprocessing, with no modification to equation, optimizer, or architecture.

Across benchmark problems and backbone choices, $\partial$-NO consistently improves prediction accuracy (\S\ref{sec:experiments}). The learned orders reflect a representation of known physics modulated by noise and finite-sample constraints, providing both better performance and interpretable insight into each problem's differential structure.

\begin{figure}[t]
\centering
\includegraphics[width=1.0\textwidth]{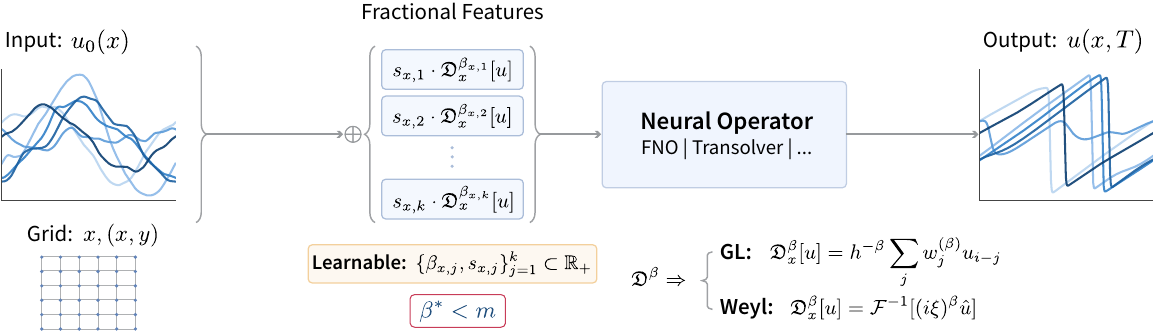}
\vspace{0.9em}
\caption{$\partial$-NO architecture.\protect\footnotemark{} The input function $u$ is augmented with learnable fractional derivatives $\Dfrak^{\beta_k} u$, where each order $\beta_k$ and scale $s_k$ is optimized jointly with backbone weights during training. This simple augmentation wraps any neural operator backbone without modification.}
\label{fig:architecture}
\end{figure}
\footnotetext{Code: \url{https://github.com/FaresBMehouachi/delNO}}

\section{Background}
\label{sec:background}

\subsection{Neural Operators}

Let $\G: \mathcal{U} \to \mathcal{V}$ denote the true solution operator between function spaces. A neural operator $\G_\theta$ is its finite-dimensional approximation parameterized by $\theta$, learned from paired examples $\{(u_i, v_i)\}_{i=1}^n$ with $v_i = \G[u_i]$ by minimizing the empirical risk:

\begin{equation}
\min_\theta \frac{1}{n}\sum_{i=1}^n \norm{\G_\theta[u_i] - v_i}_{L^2}^2.
\end{equation}
The target operator is typically a PDE solution map, taking initial or boundary conditions to the corresponding solution.

Different architectures approach this problem from different angles. The Fourier Neural Operator \citep{li2020fourier} performs convolutions efficiently in frequency space. DeepONet \citep{lu2021learning} factorizes the operator into branch and trunk networks. More recent work explores attention mechanisms \citep{hao2023gnot, wu2024Transolver, luo2025transolver++} and implicit representations \citep{serrano2023operator}. What these methods share is their input representation: they receive the input function sampled at spatial locations, either as coordinate-value pairs or on a regular grid.

\subsection{A Concrete Example: Burgers' Equation}

To make the limitation concrete, consider the viscous Burgers' equation:
\begin{equation}
\label{eq:burgers}
\partial_t u + u \, \partial_x u = \nu \, \partial_{xx} u.
\end{equation}
The solution operator $\G: u(\cdot, 0) \mapsto u(\cdot, T)$ maps initial conditions to solutions at time $T$. This operator depends on the initial condition through two mechanisms: the diffusive term $\nu \partial_{xx} u$ smooths the solution according to second derivatives, while the convective term $u \partial_x u$ transports features at speeds determined by first derivatives. A standard neural operator must learn both differentiation operations from data.

Recent work has explored providing precomputed derivatives as additional inputs to neural operators \citep{zhu2025generalizing}. This raises a natural question: which derivatives should we provide? When the PDE is known, one might supply $\partial_x u$ and $\partial_{xx} u$ to match the equation structure. But what if the governing physics is unknown or only partially specified? And even when known, is matching the PDE order necessarily optimal?

These considerations motivate a different approach: rather than prescribing derivative orders, let the network learn them from data. To enable this, we need derivatives defined for non-integer orders.

\subsection{Fractional Derivatives}

To enable optimization over derivative orders, we need a continuous parameterization. Fractional calculus provides exactly this.

\begin{definition}[Fractional Derivative]
\label{def:frac_deriv}
For order $\beta \geq 0$ and function $u \in H^s(\Omega)$ with $s > \beta$, we define the \emph{Weyl fractional derivative} through the Fourier transform:
\begin{equation}
\label{eq:frac_deriv}
D^\beta u := \mathcal{F}^{-1}\left[ (i\xi)^\beta \hat{u}(\xi) \right],
\end{equation}
where $\hat{u}(\xi) = \mathcal{F}[u](\xi)$ denotes the Fourier transform and $(i\xi)^\beta = |\xi|^\beta e^{i(\pi/2)\beta \, \mathrm{sgn}(\xi)}$.
\end{definition}

The Weyl derivative acts as a frequency filter weighting mode $\xi$ by $(i\xi)^\beta$, preserving phase information that distinguishes odd and even derivatives. When $\beta = 1$, we recover $\partial_x$; when $\beta = 2$, we get $\partial_{xx}$. Intermediate values interpolate smoothly. For non-periodic domains or to avoid Fourier artifacts (Gibbs ringing), we use the Gr\"unwald-Letnikov finite-difference formulation (Appendix~\ref{app:gl}).

Crucially, the definition is differentiable with respect to $\beta$:
\begin{equation}
\frac{\partial}{\partial \beta} (i\xi)^\beta = (i\xi)^\beta \left( \log|\xi| + i\frac{\pi}{2}\mathrm{sgn}(\xi) \right).
\end{equation}
This enables gradient-based learning of derivative orders.

\section{Picard Iteration and Derivative Emergence}
\label{sec:picard}

Why should derivative features help neural operators? The answer emerges from analyzing how PDE solutions depend on initial conditions through Picard iteration.

\subsection{Mild Solutions and the Green's Function}

Consider a semilinear PDE of the form:
\begin{equation}
\label{eq:semilinear}
\partial_t u = \Lop u + \Non[u], \quad u(0) = u_0,
\end{equation}
where $\Lop$ is a linear differential operator of order $m$ and $\Non$ is a nonlinearity (polynomial or quasilinear). The mild solution satisfies the integral equation:
\begin{equation}
\label{eq:mild}
u(t) = G_t \ast u_0 + \int_0^t G_{t-s} \ast \Non[u(s)] \, ds,
\end{equation}
where $G_t$ is the Green's function (semigroup) generated by $\Lop$. For constant-coefficient operators, this takes the explicit form $G_t = e^{t\Lop}$.

\subsection{The Picard Mechanism}

Picard iteration constructs solutions by successive approximation. Starting from $u^{(0)}(t) = G_t \ast u_0$, each iterate refines the approximation:
\begin{equation}
\label{eq:picard}
u^{(k+1)}(t) = G_t \ast u_0 + \int_0^t G_{t-s} \ast \Non[u^{(k)}(s)] \, ds.
\end{equation}
The key observation is that the first iterate already reveals essential structure.

Expanding the semigroup for small times, $G_s \ast u_0 \approx u_0 + s \Lop u_0 + O(s^2)$, and substituting into the nonlinearity:
\begin{equation}
\Non[G_s \ast u_0] \approx \Non[u_0] + s \cdot D\Non[u_0][\Lop u_0] + O(s^2),
\end{equation}
where $D\Non[u_0][v]$ is the Fr\'echet derivative of $\Non$ at $u_0$ applied to direction $v$. Evaluating the integral yields:

\begin{theorem}[Derivative Emergence, Parabolic Case]
\label{thm:derivative_emergence}
Let $\G: u_0 \mapsto u(T)$ be the solution operator for a semilinear parabolic PDE with polynomial nonlinearity $\Non$ of degree $N$ and linear operator $\Lop$ of differential order $m$. For $u_0 \in B_s(R)$, there exists $T_*(R) = (L(R))^{-1}$, where $L(R)$ is the Lipschitz constant of $\Non$ on $B_s(R)$, such that for all $T < T_*(R)$:
\begin{equation}
\G[u_0] = G_T \ast u_0 + T \cdot \bar{G}_T \ast \Non[u_0] + \mathcal{R}(u_0, T),
\end{equation}
where $\bar{G}_T = T^{-1}\int_0^T G_{T-s} \, ds$ is an averaged Green's function, and the remainder satisfies $\|\mathcal{R}(u_0, T)\|_{L^2} \leq C(R) T^2$ uniformly over $u_0 \in B_s(R)$. The nonlinearity $\Non[u_0]$ involves derivatives of $u_0$ up to order $m$ and powers up to $u_0^N$.
\end{theorem}

The full proof appears in Appendix~\ref{app:picard}. The essential point is that derivatives enter through two mechanisms: the linear operator $\Lop$ directly differentiates $u_0$, while the nonlinearity $\Non$ combines powers and derivatives according to its structure.

\begin{remark}[Beyond Short Times]
The derivative dependence persists for arbitrary times via iterated Picard and extends to stationary PDEs (e.g., Darcy flow) and quasilinear systems (e.g., Navier-Stokes) where derivatives appear directly. See Appendix~\ref{app:picard_extended}.
\end{remark}

\begin{remark}[Beyond Parabolic Evolution]
\label{rem:beyond_parabolic}
A complementary spectral resolvent argument (Appendix~\ref{app:spectral_general}) extends the same derivative structure to general evolution equations $\partial_t^a u = \Lop u$ with temporal order $a > 0$, recovering parabolic ($a{=}1$), hyperbolic ($a{=}2$), and fractional-in-time (Mittag-Leffler) cases uniformly. The temporal order $a$ determines the number of initial conditions $\left\lceil a\right\rceil $ but not the derivative-feature structure, which is set by the spatial order $m$.
\end{remark}

\begin{example}[Burgers Features]
For Burgers' equation, $\Lop = \nu \partial_{xx}$ and $\Non[u] = -u \partial_x u$. The solution operator depends on:
\begin{equation}
\{u_0, \, \partial_x u_0, \, \partial_{xx} u_0, \, u_0 \cdot \partial_x u_0\}.
\end{equation}
The first three are derivatives; the fourth emerges from the convective nonlinearity. In this work, we focus on the \emph{derivative terms} $\{u_0, \partial_x u_0, \partial_{xx} u_0\}$, leaving the nonlinear product $u_0 \cdot \partial_x u_0$ to the network. This differs from equation-aware approaches that encode full nonlinear structure \citep{zhu2025generalizing}; our method requires only the PDE order $m$, not equation-specific terms.
\end{example}

\section{Approximation Theory: Why Features Help}
\label{sec:approx}

We now prove that derivative features admit a better approximation bound. This section establishes that providing derivative features is beneficial; the \emph{optimal} choice of derivative order is deferred to Section~\ref{sec:optimal_beta}. The improvement quantifies how much easier the learning problem becomes when inputs align with PDE structure.

\subsection{The Feature Map}

\begin{definition}[Feature Maps]
\label{def:feature_map}
We define two related feature maps. For a single derivative order $\beta \geq 0$, the \emph{$\beta$-feature map} $\Phi^\beta: H^s(\Omega) \to H^s \times H^{s-\beta}$ is:
\begin{equation}
\Phi^\beta[u] = (u, D^\beta u).
\end{equation}
For the full set of integer derivatives up to order $m$, the \emph{full feature map} $\Phi^{[m]}: H^s(\Omega) \to \mathcal{X}$ is:
\begin{equation}
\Phi^{[m]}[u] = (u, D^1 u, D^2 u, \ldots, D^m u),
\end{equation}
where $\mathcal{X} = H^s \times H^{s-1} \times \cdots \times H^{s-m}$ with the product norm. When context is clear, we write $\Phi$ for either map.
\end{definition}

The feature map bundles the input function with its derivatives. Since $\Phi$ includes $u$ as its first component, no information is lost.

\begin{assumption}[Regularity]
\label{ass:regularity}
Inputs lie in a Sobolev ball: $u \in B_s(R) := \{u \in H^s(\Omega) : \norm{u}_{H^s} \leq R\}$ with $s > m + d/2$.
\end{assumption}

\begin{proposition}[Stability]
\label{prop:stability}
Under Assumption~\ref{ass:regularity}, the feature map $\Phi: B_s(R) \to \mathcal{X}$ is bi-Lipschitz:
\begin{equation}
\|u_1 - u_2\|_{H^s} \leq \|\Phi[u_1] - \Phi[u_2]\|_{\mathcal{X}} \leq \sqrt{1+m} \cdot \|u_1 - u_2\|_{H^s}.
\end{equation}
The upper constant $\sqrt{1+m}$ is independent of the ball radius $R$.
\end{proposition}

The lower bound is immediate since $\Phi$ preserves $u$. The upper bound follows from the continuity of differentiation in Sobolev spaces. See Appendix~\ref{app:stability} for the full proof.

\subsection{Main Approximation Result}

\begin{theorem}[Approximation Rates]
\label{thm:approx}
Let $\G: B_s(R) \to L^2(\Omega)$ be the solution operator for a PDE of differential order $m$. Denote by $\NO_W$ the class of neural operators with $W$ parameters.

\textbf{(i) Without features:}
\begin{equation}
\inf_{\G_\theta \in \NO_W} \sup_{u \in B_s(R)} \norm{\G[u] - \G_\theta(u)}_{L^2} = O\left(W^{-\frac{s-m}{d}}\right).
\end{equation}

\textbf{(ii) With full features $\Phi^{[m]}$:}
\begin{equation}
\inf_{\G_\theta \in \NO_W} \sup_{u \in B_s(R)} \norm{\G[u] - \G_\theta(\Phi^{[m]}[u])}_{L^2} = O\left(W^{-\frac{s}{d}}\right).
\end{equation}
\end{theorem}

\begin{proof}
Both bounds are achievability results (upper bounds on the minimax error). Full details appear in Appendix~\ref{app:approx}, with the scope of the comparison between (i) and (ii) discussed in Remark~\ref{rem:scope}.

\textbf{Part (i):} The operator $\G: B_s(R) \to L^2$ has modulus of continuity $\omega(\varepsilon) \leq C(R) \varepsilon^{(s-m)/s}$ when measured in $L^2$. This follows from Sobolev interpolation: for $u_1, u_2 \in B_s(R)$ with $\|u_1 - u_2\|_{L^2} \leq \varepsilon$, we have $\|u_1 - u_2\|_{H^m} \lesssim \varepsilon^{(s-m)/s} R^{m/s}$. The H\"older exponent $(s-m)/s$ arises because inputs lie in $H^s$ but the operator depends on $m$-th order derivatives: the ``missing'' regularity $m$ out of total $s$ determines the continuity loss. The $\epsilon$-covering number of $B_s(R)$ in $L^2$ satisfies $\log \mathcal{N}(\epsilon) \asymp (R/\epsilon)^{d/s}$ \citep{carl1990entropy}. Applying the approximation theory of \citet{yarotsky2017error} for H\"older-continuous functions on sets with this metric entropy yields the rate $W^{-(s-m)/d}$.

\textbf{Part (ii):} Define the lifted operator $\tilde{\G}[\Phi^{[m]}[u]] := \G[u]$. On bounded balls, $\tilde{\G}$ is Lipschitz: since the features $\Phi^{[m]}[u] = (u, D^1 u, \ldots, D^m u)$ provide the derivatives explicitly, the operator no longer needs to ``invert'' differentiation, an unbounded operation that was the source of the H\"older discontinuity in Part (i). By Theorem~\ref{thm:derivative_emergence}, the solution operator depends on $u$ and its derivatives through polynomial combinations; polynomial functions are Lipschitz on bounded balls with constant $L(R) = O(R^{N-1})$ for degree-$N$ nonlinearity. By Proposition~\ref{prop:stability}, the feature map is bi-Lipschitz, so metric entropy is preserved up to constants. Applying \citet{yarotsky2017error} to the Lipschitz function $\tilde{\G}$ achieves rate $W^{-s/d}$.
\end{proof}The improvement in the bound, from $(s-m)/d$ to $s/d$, is substantial. For a second-order PDE in one dimension with $s=3$, the exponent improves from 1 to 3, meaning the same accuracy requires far fewer parameters.

\section{The Statistical Twist: Why $\beta^* < m$}
\label{sec:optimal_beta}

Theorem~\ref{thm:approx} suggests using the true PDE order $m$ as the derivative order. We now prove this is suboptimal for noisy, finite samples. This finding is our central theoretical contribution.

\begin{remark}[Complementary Theoretical Perspectives]
\label{rem:frameworks}
Section~\ref{sec:approx} and this section answer \emph{different questions}. Section~\ref{sec:approx} asks: ``Do derivative features help?'' and answers via approximation theory. This section asks: ``Given that features help, what order is optimal?'' and answers via spectral bias-variance analysis. These perspectives are complementary: one justifies using derivative features at all, the other determines the best choice.
\end{remark}

\subsection{The Core Tension}

Consider learning an operator that depends on $D^m u$ using features $D^\beta u$ with potentially different order $\beta$. Two competing effects arise:

\textbf{Bias:} If $\beta < m$, the features miss high-frequency information relevant to the target. This creates approximation error that persists regardless of sample size.

\textbf{Variance:} Differentiation amplifies noise. The Fourier multiplier $|\xi|^\beta$ grows unboundedly with frequency, meaning high-frequency noise gets amplified by factor $|\xi|^\beta$ when we compute $D^\beta u$. Higher $\beta$ means more amplification.

These effects respond oppositely to $\beta$: bias decreases while variance increases. The optimal choice balances them.

\subsection{Spectral Analysis}

To make this tradeoff precise, we assume inputs have a typical spectral structure.

\begin{assumption}[Spectral Decay]
\label{ass:spectral}
The input distribution $\mu$ on $B_s(R)$ has power spectral density $S(\xi) := \E[|\hat{u}(\xi)|^2] = C_S |\xi|^{-2s-d}$ for $1 \leq |\xi| \leq \xi_{\max}$, where $C_S > 0$ depends on the distribution.
\end{assumption}

This rate is the borderline upper envelope for $H^s$ integrability: $\E\|u\|_{H^s}^2$ converges for any $S(\xi)$ decaying strictly faster than $|\xi|^{-2s-d}$ and is logarithmically divergent at exactly this rate. We adopt this envelope as the worst-case spectral profile on $B_s(R)$, yielding the worst-case bias integral in \S\ref{sec:optimal_beta}; Mat\'ern-$s$ fields \citep{stein1999interpolation} sit at this edge. The coefficient $C_S$ cancels in the optimal $\beta^*$ formula (Theorem~\ref{thm:optimal_beta}), so only the exponent matters.

\begin{proposition}[Bias-Variance Decomposition]
\label{prop:bias_variance}
Under Assumption~\ref{ass:spectral}, suppose the network has effective frequency cutoff $\xi_c$ (depending on capacity $W$ but not on $\beta$). Using features $D^\beta u$ with $\beta \leq m$, the MSE decomposes as $\MSE(\beta) = B(\beta) + V(\beta)$ where:
\begin{align}
\text{Bias:} \quad B(\beta) &= C_B \int_{|\xi| > \xi_c} |\xi|^{2(m-\beta)} S(\xi) \, d\xi, \\
\text{Variance:} \quad V(\beta) &= \frac{\sigma^2 C_V}{n} \int_{|\xi| \leq \xi_{\max}} |\xi|^{2\beta} \, d\xi,
\end{align}
where $\xi_{\max} = \pi/h$ is the Nyquist frequency for grid spacing $h$.

The bias captures the spectral mismatch: providing $D^\beta u$ instead of $D^m u$ leaves a residual $|\xi|^{m-\beta}$ that the network must learn implicitly. The variance captures noise amplification: computing $D^\beta u$ amplifies noise at frequency $\xi$ by factor $|\xi|^\beta$.

Both integrals admit closed forms (see Appendix~\ref{app:bias_variance}). The bias is strictly decreasing in $\beta$. The variance is strictly increasing in $\beta$ provided $\xi_{\max} > e^{1/(2\beta+d)}$, which holds for typical discretizations with $\xi_{\max} \gg 1$. There exists a unique $\beta^* \in [0, m]$ minimizing the total MSE.
\end{proposition}

\begin{remark}[Terminology]
\label{rem:terminology}
We use ``bias'' and ``variance'' by analogy with statistical learning theory. Here, $B(\beta)$ represents \emph{approximation bias} from using derivative order $\beta < m$: high-frequency information beyond network capacity cannot be captured. Meanwhile, $V(\beta)$ represents \emph{estimation variance} from noise amplification: computing $D^\beta u$ from noisy data amplifies observation noise by $|\xi|^\beta$ at each frequency. This mirrors the classical tradeoff where model complexity reduces bias at the cost of increased variance.
\end{remark}

\begin{remark}[Capacity Cutoff in FNO]
\label{rem:capacity_fno}
For Fourier Neural Operators, the capacity cutoff is well approximated by $\xi_c = 2\pi k_{\max}/L$: FNO truncates to $k_{\max}$ Fourier modes by architectural design, imposing a hard frequency cutoff independent of input features. The preprocessing $D^\beta u$ changes the input's spectral content but not the architecture's representational capacity.
\end{remark}

\subsection{The Optimal Derivative Order}

Balancing bias and variance yields our main theoretical result.

\begin{theorem}[Optimal Derivative Order]
\label{thm:optimal_beta}
Under Assumptions~\ref{ass:regularity} and~\ref{ass:spectral}, define $\tau := 2(s-m) > 0$ and the capacity constant $A_c := 2s \ln\xi_c + (2m+d)\ln\xi_{\max}$, where $\xi_c$ is the network's frequency cutoff and $\xi_{\max} = \pi/h$ is the Nyquist frequency. The MSE-optimal derivative order satisfies:
\begin{equation}
\label{eq:optimal_beta}
\boxed{\beta^* = m - \delta^*, \quad \delta^* = \frac{A_c - \ln(n/\sigma^2)}{2\ln(\xi_c \cdot \xi_{\max})}}
\end{equation}
up to $O(1/\ln\xi_c)$ corrections, with $\beta^* \in [0, m]$.
\end{theorem}

The formula shows that $\delta^*$ depends on the signal-to-noise ratio $n/\sigma^2$: as data quality improves, the gap shrinks and $\beta^* \to m$. The capacity constant $A_c$ captures network architecture effects. The full derivation appears in Appendix~\ref{app:optimal_beta}.
\begin{remark}[Data Sufficiency for Exact Physics]
\label{rem:data_sufficiency}
To leading order in $1/\ln \xi_c$, $\delta^* > 0$ strictly whenever $n/\sigma^2 < e^{A_c}$, with exact-physics recovery requiring
\begin{equation}
n^\star = \sigma^2 \exp(A_c).
\end{equation}
The exponential dependence makes $n^\star$ astronomical in practice. For a representative FNO setting (unit-interval Burgers, $N{=}1024$, $k_{\max}{=}16$, giving $\xi_{\max} \approx 3200$, $\xi_c \approx 100$, $A_c \approx 68$), even at $\sigma = 10^{-3}$ this implies $n^\star \approx 10^{23}$ samples (a leading-order figure, shifted by the unidentified constant $\kappa$ of App.~\ref{app:delta_next_order}), far beyond any realistic training budget. Sub-physical regularization is therefore the operational regime at any practically achievable scale.
\end{remark}
These results have three key implications:
\begin{enumerate}[leftmargin=*,itemsep=2pt]
\item \textbf{Strict suboptimality of $m$:} For any finite sample size $n$ and noise level $\sigma > 0$ with $n/\sigma^2 < e^{A_c}$, the optimal order satisfies $\beta^* < m$ strictly. Below this threshold, the exact PDE order is never optimal.

\item \textbf{Asymptotic recovery:} As $n \to \infty$ (or $n \geq n^\star$) or $\sigma \to 0$, the gap $\delta^* \to 0$ and $\beta^* \to m$. Physics is recovered in the large-data limit.

\item \textbf{Capacity-adaptive:} Networks with higher frequency cutoff $\xi_c$ (e.g., more Fourier modes in FNO) benefit from lower $\beta^*$, automatically trading approximation power for regularization.
\end{enumerate}

\section{$\partial$-NO: Derivative-Augmented Neural Operators}
\label{sec:method}

Our theoretical analysis motivates a simple architecture: neural operators with learnable fractional derivative features (Figure~\ref{fig:architecture}).

Given a base neural operator $\NO_\theta$, the $\partial$-NO augmentation replaces inputs $(x, u)$ with:
\begin{equation}
\partial\text{-NO}: \quad \left(x, u, s_1 \Dfrak^{\beta_1} u, s_2 \Dfrak^{\beta_2} u, \ldots, s_k \Dfrak^{\beta_k} u\right),
\end{equation}
where both the derivative orders $\beta_1, \ldots, \beta_k$ and the scales $s_1, \ldots, s_k$ are learnable parameters optimized jointly with network weights $\theta$. The learnable scales normalize features across different derivative orders, since $\Dfrak^\beta u$ can vary dramatically in magnitude as $\beta$ changes.

While the Weyl fractional derivative (Definition~\ref{def:frac_deriv}) provides an elegant spectral formulation, it can introduce Fourier artifacts (Gibbs ringing) near discontinuities. We instead use the Gr\"unwald-Letnikov (GL) finite-difference formulation:
\begin{equation}
\Dfrak^\beta u(x) = \frac{1}{h^\beta} \sum_{j=0}^{K_{\mathrm{GL}}} (-1)^j \binom{\beta}{j} u(x - jh),
\end{equation}
where $h$ is the grid spacing and the generalized binomial coefficients $\binom{\beta}{j}$ extend naturally to non-integer $\beta$ via the Gamma function. This formulation is differentiable with respect to $\beta$, enabling gradient-based optimization.

In practice, we omit the factor $h^{-\beta}$. For small $h$ and moderate $\beta$, this factor can be dangerously large and destabilize training. Since the learnable scales $s_k$ and the network itself absorb any constant scaling, removing $h^{-\beta}$ preserves the functional structure of the derivative while ensuring training stability. We denote backbones augmented with learnable derivative features using the $\partial$-prefix: \textbf{$\partial$-FNO}, \textbf{$\partial$-Transolver}, etc. %

\section{Experiments}
\label{sec:experiments}

We validate $\partial$-NO across multiple PDEs and backbone architectures. Our experiments address two questions: Does $\partial$-augmentation improve accuracy? Does the theory correctly predict how $\beta^*$ varies with noise, sample size, and network capacity?

\subsection{Setup}

\textbf{Datasets.} We use benchmark datasets following \citet{li2020fourier}: Burgers' equation (smooth, $\nu = 0.1$; near-shock, $\nu = 0.001$), Darcy flow (elliptic), and two-dimensional Navier-Stokes (incompressible turbulence at $\nu = 10^{-5}$), plus a third-order KdV dataset ($m{=}3$) obtained by direct numerical integration from random initial conditions (App.~\ref{app:experiments}).

\textbf{Backbones.} Our primary goal is to validate the theoretical predictions from Section~\ref{sec:optimal_beta}. Neural operator training is particularly sensitive to hyperparameters, often requiring very small batch sizes (values of 8, 4, or even 1--2 are common \citep{wu2024Transolver}) to achieve good performance. This highlights the delicate nature of training such models. We use FNO \citep{li2020fourier} as a canonical spectral method, its tensorized variant TFNO, the localized operator LocalNO \citep{liu2024neural}, and Transolver \citep{wu2024Transolver} as a state-of-the-art attention-based architecture, employing the THUML library \citep{thuml2024neuraloperator} with default hyperparameters.

\begin{table}[!h]
\caption{Relative $L^2$ error (\%), averaged over 5 seeds, across 5 PDE benchmarks $\times$ 4 backbones $\times$ \{clean, noisy\}. Bold = best per (backbone, dataset, noise) cell. $\partial$-augmentation improves every cell tested.}
\label{tab:main}
\vskip 0.05in
\centering
\resizebox{\textwidth}{!}{%
\begin{tabular}{@{}l cc cc cc cc cc@{}}
\toprule
& \multicolumn{2}{c}{\textbf{Burgers} ($\nu{=}0.1$)} & \multicolumn{2}{c}{\textbf{Burgers} ($\nu{=}0.001$)} & \multicolumn{2}{c}{\textbf{KdV}} & \multicolumn{2}{c}{\textbf{Darcy} (Elliptic)} & \multicolumn{2}{c}{\textbf{NS2D} ($\nu{=}10^{-5}$)} \\
\cmidrule(lr){2-3}\cmidrule(lr){4-5}\cmidrule(lr){6-7}\cmidrule(lr){8-9}\cmidrule(lr){10-11}
& $\sigma{=}0$ & $\sigma{=}0.05$ & $\sigma{=}0$ & $\sigma{=}0.05$ & $\sigma{=}0$ & $\sigma{=}0.05$ & $\sigma{=}0$ & $\sigma{=}0.05$ & $\sigma{=}0$ & $\sigma{=}0.05$ \\
\midrule
FNO              & 0.065 & 0.250 & 1.138 & 1.632 & 1.598 & 2.255 & 5.643 & 5.590 & 3.479 & 3.575 \\
$\partial$-FNO   & \textbf{0.055} & \textbf{0.190} & \textbf{0.837} & \textbf{1.366} & \textbf{1.549} & \textbf{2.195} & \textbf{5.426} & \textbf{5.415} & \textbf{3.428} & \textbf{3.511} \\
\rowcolor{gray!10} \textit{Improv.} & 15.1\% & 24.1\% & 26.5\% & 16.3\% & 3.0\% & 2.7\% & 3.8\% & 3.1\% & 1.5\% & 1.8\% \\
\midrule
TFNO             & 0.071 & 0.209 & 0.702 & 1.032 & 2.312 & 2.745 & 5.223 & 5.206 & 3.428 & 3.534 \\
$\partial$-TFNO  & \textbf{0.062} & \textbf{0.178} & \textbf{0.676} & \textbf{0.917} & \textbf{2.227} & \textbf{2.660} & \textbf{4.843} & \textbf{4.906} & \textbf{3.228} & \textbf{3.368} \\
\rowcolor{gray!10} \textit{Improv.} & 11.9\% & 14.7\% & 3.7\% & 11.2\% & 3.7\% & 3.1\% & 7.3\% & 5.8\% & 5.8\% & 4.7\% \\
\midrule
LocalNO          & 0.063 & 0.236 & 0.768 & 1.272 & 2.587 & 3.625 & 5.511 & 5.534 & 1.994 & 2.620 \\
$\partial$-LocalNO & \textbf{0.060} & \textbf{0.199} & \textbf{0.732} & \textbf{1.203} & \textbf{2.352} & \textbf{3.441} & \textbf{5.430} & \textbf{5.487} & \textbf{1.984} & \textbf{2.533} \\
\rowcolor{gray!10} \textit{Improv.} & 6.1\% & 15.8\% & 4.7\% & 5.4\% & 9.1\% & 5.1\% & 1.5\% & 0.9\% & 0.5\% & 3.3\% \\
\midrule
Transolver       & 0.990 & 0.765 & 18.988 & 7.521 & 22.427 & 16.334 & 14.645 & 15.332 & 4.301 & 4.821 \\
$\partial$-Transolver & \textbf{0.709} & \textbf{0.736} & \textbf{4.633} & \textbf{4.049} & \textbf{17.957} & \textbf{9.588} & \textbf{13.125} & \textbf{11.212} & \textbf{4.043} & \textbf{4.341} \\
\rowcolor{gray!10} \textit{Improv.} & 28.4\% & 3.8\% & 75.6\% & 46.2\% & 19.9\% & 41.3\% & 10.4\% & 26.9\% & 6.0\% & 9.9\% \\
\bottomrule
\end{tabular}%
}
\end{table}

\textbf{Derivative features.} For the second-order benchmarks (Burgers, Darcy, NS), $m=2$ and we use two learnable derivative features. For 1D problems: $\Dfrak_x^{\beta_1} u$ and $\Dfrak_x^{\beta_2} u$, initialized at $\beta_1 = 1.0$ and $\beta_2 = 2.0$; for the third-order KdV ($m=3$) we use a third feature $\Dfrak_x^{\beta_3} u$ initialized at $\beta_3 = 3.0$. For 2D problems, we include directional derivatives $\Dfrak_x^{\beta_{x,1}} u$, $\Dfrak_x^{\beta_{x,2}} u$, $\Dfrak_y^{\beta_{y,1}} u$, $\Dfrak_y^{\beta_{y,2}} u$ and one mixed partial $\Dfrak_y^{\beta_{y,c}} \Dfrak_x^{\beta_{x,c}} u$. Each order is learned independently. In particular, the mixed partial has two learnable orders ($\beta_{x,c}, \beta_{y,c}$) and a single shared scaling factor. %

\textbf{Training.} All models train with AdamW. The $\beta$ parameters use $10\times$ higher learning rate ($10^{-2}$ vs $10^{-3}$) because they sit at the input level: gradients through them propagate back through the full backbone and attenuate accordingly. We report relative $L^2$ errors (\%). Full hyperparameters and implementation details are provided in Appendix~\ref{app:experiments}.

\subsection{Validation of $\partial$-NO Improvements}

Table~\ref{tab:main} summarizes our main results: five PDE benchmarks, four backbones, and both clean ($\sigma{=}0$) and noisy ($\sigma{=}0.05$) regimes.
The $\partial$-augmentation improves every backbone/dataset/noise combination tested; the separate CNO benchmarks (\S\ref{sec:cno_experiments}, Table~\ref{tab:cno}) show larger typical gains. A broader stress test under Gaussian and uniform additive noise with $\sigma \in [0.05, 0.3]$ confirms this robustness across every tested setting (Appendix~\ref{app:stress}), and a matched-parameter control isolates the gains to the derivative inductive bias rather than added capacity (Appendix~\ref{app:matched}). 
Further robustness analyses (emergent directional isotropy on 2D/3D isotropic PDEs in Figs.~\ref{fig:isotropy},~\ref{fig:cns_3d}; resolution-dependent gains in Fig.~\ref{fig:resolution}; complementarity with PINO physics-informed training in Table~\ref{tab:pino}) are reported in the appendix.

\subsection{Theory Validation}
\label{sec:theory_validation}

Theorem~\ref{thm:optimal_beta} makes testable predictions about how $\beta^*$ responds to noise $\sigma$, sample size $n$, and network capacity (Fourier modes $k$). Since $\xi_c = 2\pi k$, the formula simplifies to show explicit dependence on the three experimental parameters:
\begin{equation}
\label{eq:delta_simplified}
\delta^* \;\propto\; \frac{2s \ln k + 2\ln\sigma - \ln n + C}{2\ln k + C'}
\end{equation}
where $C, C'$ are problem-dependent constants. This yields three monotonicities under Assumption~\ref{ass:regularity}: $\sigma \uparrow \Rightarrow \beta^* \downarrow$, $n \uparrow \Rightarrow \beta^* \uparrow$, and $k \uparrow \Rightarrow \beta^* \downarrow$. Near the smoothness boundary $s \approx m+d/2$, the capacity-effect margin shrinks and the prediction loses sharpness, as detailed below. We test all three on Burgers (smooth and shock), using a single learnable $\beta$ and varying one factor at a time.
\begin{figure}[!h]
\centering
\includegraphics[width=1.0\textwidth]{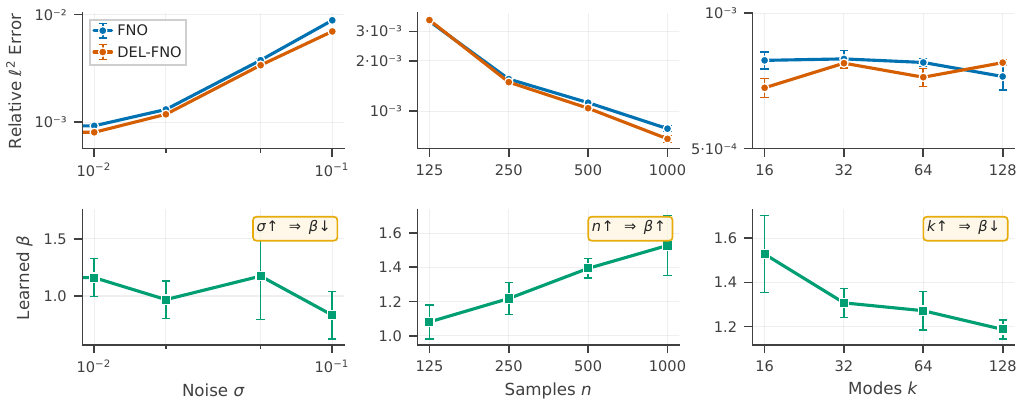}
\caption{\textbf{Burgers (smooth, $\nu{=}0.1$): Theory vs.\ experiment.} Top: relative $L^2$ error; bottom: learned $\beta^*$. Theorem~\ref{thm:optimal_beta}'s two unconditional predictions are supported: (left) $\sigma \uparrow \Rightarrow \beta^* \downarrow$; (center) $n \uparrow \Rightarrow \beta^* \uparrow$, with $\beta^*$ approaching but not reaching $m{=}2$, consistent with the unreachable recovery threshold (Remark~\ref{rem:data_sufficiency}). At smoothness $\tau \approx 2.2$ (comfortably above the Assumption~\ref{ass:regularity} threshold $s > m + d/2$), the capacity formula predicts $k \uparrow \Rightarrow \beta^* \downarrow$ (right), which holds monotonically.}
\label{fig:burgers_ablation}
\end{figure}

\begin{figure}[t!]
\centering
\includegraphics[width=1.0\textwidth]{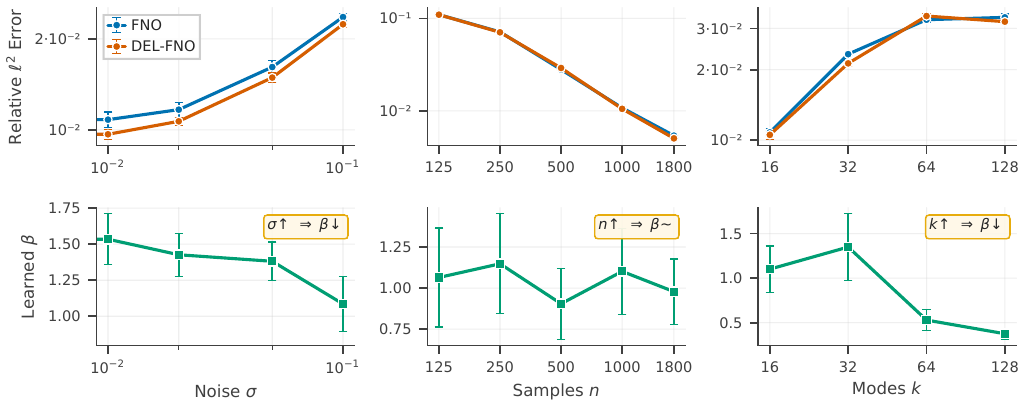}
\caption{\textbf{Burgers (near-shock, $\nu{=}0.001$): Theory vs.\ experiment.} Same panels as Figure~\ref{fig:burgers_ablation}. The noise prediction ($\sigma\uparrow\Rightarrow\beta^*\downarrow$) holds clearly; the sample-size effect ($n\uparrow\Rightarrow\beta^*\uparrow$) attenuates to a noisy plateau near the smoothness boundary ($\tau\approx 0.8$), as expected when Assumption~\ref{ass:regularity} is marginal. At rough-data smoothness $\tau \approx 0.8$ (near the Assumption~\ref{ass:regularity} smoothness boundary), the capacity formula admits a weakened $k$-effect; the empirical $k$-sweep is correspondingly noisier and non-monotone rather than the clean monotone of the smooth case, consistent with this regime change.}
\label{fig:shock_ablation}
\end{figure}

Figures~\ref{fig:burgers_ablation} and~\ref{fig:shock_ablation} confirm both unconditional predictions on smooth Burgers, the noise prediction on shock Burgers (whose sample-size effect plateaus), and the capacity effect (a decrease of $\beta^*$ with $k$ in both regimes, less cleanly so near the smoothness boundary):

\textbf{Noise ($\sigma \uparrow \Rightarrow \beta^* \downarrow$).} The $+2\ln\sigma$ term in \eqref{eq:delta_simplified} predicts that increasing noise increases $\delta^*$. Experimentally, learned $\beta^*$ decreases with $\sigma$ on both smooth and shock Burgers (monotonically on shock, a noisier downward trend on smooth), as predicted. The network automatically regularizes by learning lower derivative orders when data is noisy.

\textbf{Sample size ($n \uparrow \Rightarrow \beta^* \uparrow$).} The $-\ln n$ term predicts that more data decreases $\delta^*$ and drives $\beta^* \to m$. Experimentally, on smooth Burgers, learned $\beta^*$ increases monotonically with $n$, approaching but not reaching $m = 2$ at any practically achievable $n$. The persistent gap $\delta^* > 0$ is consistent with Remark~\ref{rem:data_sufficiency}: the recovery threshold $n^\star = \sigma^2 \exp(A_c)$ is far beyond the sample sizes accessible in this sweep.

\textbf{Capacity ($k \uparrow \Rightarrow \beta^* \downarrow$).} The $+2s\ln k$ term in the numerator competes with $+2\ln k$ in the denominator. Under Assumption~\ref{ass:regularity} ($s > m + d/2$), the numerator dominates and $k \uparrow \Rightarrow \delta^* \uparrow$; near the smoothness boundary the margin shrinks and the prediction loses sharpness. This predicts qualitatively different behavior for smooth vs.\ shock Burgers. Experimentally, we observe two qualitatively consistent but differently scaled behaviors in Figures~\ref{fig:burgers_ablation} and~\ref{fig:shock_ablation}: smooth Burgers ($\tau \approx 2.2$) shows a clean monotonic decrease, while shock Burgers ($\tau \approx 0.8$) shows a noisier, non-monotone decrease over the same range. The consistency across smooth and shock regimes supports the bias-variance framework.

 \subsection{Differential Inductive Bias: $\partial$-CNO}
\label{sec:cno_experiments}

\begin{plainblock}
A natural question is whether $\partial$-NO's gains come from an increased local processing capacity or from the specific inductive bias from fractional differentiation. Convolutional architectures provide a sharp test: a generic convolution \emph{can} represent a finite-difference derivative, but it is not \emph{constrained} to learn one. To validate the importance of the $\partial$-NO's inductive bias, we instantiate $\partial$-augmentation on the Convolutional Neural Operator \citep{raonic2023convolutional} (CNO) and compare $\partial$-CNO against the vanilla CNO across seven PDE benchmarks from the benchmark suite of \citet{raonic2023convolutional}.

\begin{table}[t!]
\caption{$\partial$-CNO vs.\ CNO on seven benchmarks from the suite of \citet{raonic2023convolutional}, evaluated under both relative $L^2$ error and median relative $L^1$ error (averaged over 5 seeds). Bold = best per (dataset, noise, metric) cell. $\partial$-CNO improves CNO in $\mathbf{14/14}$ settings under \emph{both} metrics.}
\label{tab:cno}
\vskip 0.05in
\centering
\resizebox{\textwidth}{!}{%
\begin{tabular}{@{}cl cc cc cc cc cc cc cc@{}}
\toprule
& & \multicolumn{2}{c}{\textbf{Poisson}} & \multicolumn{2}{c}{\textbf{NS Shear}} & \multicolumn{2}{c}{\textbf{Wave}} & \multicolumn{2}{c}{\textbf{Allen--Cahn}} & \multicolumn{2}{c}{\textbf{Transp.\ Cont.}} & \multicolumn{2}{c}{\textbf{Darcy}} & \multicolumn{2}{c}{\textbf{Airfoil}} \\
\cmidrule(lr){3-4}\cmidrule(lr){5-6}\cmidrule(lr){7-8}\cmidrule(lr){9-10}\cmidrule(lr){11-12}\cmidrule(lr){13-14}\cmidrule(lr){15-16}
& $\sigma$ & 0 & 0.05 & 0 & 0.05 & 0 & 0.05 & 0 & 0.05 & 0 & 0.05 & 0 & 0.05 & 0 & 0.05 \\
\midrule
\multirow{3}{*}{\rotatebox{90}{\textbf{$L^2$}}}
  & CNO              & 1.788 & 3.664 & 9.686 & 12.218 & 0.590 & 1.990 & 2.618 & 1.980 & 0.416 & 0.413 & 0.763 & 0.724 & 1.184 & 1.130 \\
  & $\partial$-CNO   & \textbf{1.053} & \textbf{1.293} & \textbf{7.104} & \textbf{7.746} & \textbf{0.520} & \textbf{0.900} & \textbf{1.632} & \textbf{1.575} & \textbf{0.353} & \textbf{0.381} & \textbf{0.752} & \textbf{0.718} & \textbf{1.137} & \textbf{1.111} \\
\rowcolor{gray!10}
  & \textit{Improv.} & \impv{41.2\%} & \impv{64.7\%} & \impv{26.7\%} & \impv{36.6\%} & \impv{11.8\%} & \impv{54.8\%} & \impv{37.7\%} & \impv{20.4\%} & \impv{15.2\%} & \impv{7.7\%} & \impv{1.4\%} & \impv{0.8\%} & \impv{4.0\%} & \impv{1.7\%} \\
\midrule
\multirow{3}{*}{\rotatebox{90}{\textbf{$L^1$}}}
  & CNO              & 1.200 & 2.504 & 8.120 & 10.080 & 0.440 & 1.550 & 0.870 & 0.940 & 0.358 & 0.380 & 0.514 & 0.482 & 0.446 & 0.473 \\
  & $\partial$-CNO   & \textbf{0.660} & \textbf{0.884} & \textbf{5.910} & \textbf{6.110} & \textbf{0.439} & \textbf{0.692} & \textbf{0.646} & \textbf{0.857} & \textbf{0.341} & \textbf{0.352} & \textbf{0.501} & \textbf{0.457} & \textbf{0.443} & \textbf{0.410} \\
\rowcolor{gray!10}
  & \textit{Improv.} & \impv{45.0\%} & \impv{64.7\%} & \impv{27.2\%} & \impv{39.4\%} & \impv{0.2\%} & \impv{55.4\%} & \impv{25.7\%} & \impv{8.8\%} & \impv{4.7\%} & \impv{7.5\%} & \impv{2.6\%} & \impv{5.1\%} & \impv{0.7\%} & \impv{13.3\%} \\
\bottomrule
\end{tabular}%
}
\end{table}
Results are in Table~\ref{tab:cno}. $\partial$-CNO improves CNO in all settings, with peak gains up to \textbf{55--65\%} on Poisson and Wave, and $\sim 39\%$ on NS Shear. The pattern is consistent with the differential interpretation: gains are largest on PDEs whose solution operators most depend on global differential structure (elliptic Poisson, shear-dominated NS, wave propagation), and smaller on data-dominated settings (Darcy, Airfoil). Because CNO is built entirely from learnable convolutions, the persistent improvement from imposing the GL derivative structure on the input is direct evidence that derivative features carry information convolutional layers do not reliably acquire from data alone.
\end{plainblock}

\subsection{Robustness and Directionality}
\label{sec:robustness}
\begin{plainblock}
Beyond \S\ref{sec:theory_validation}'s monotonicities, the bias-variance picture predicts three further effects: learnable orders should recover the geometric symmetries of the PDE without prescription, $\partial$-NO's advantage should grow at coarser resolution, and input-side derivative features should complement loss-side PDE residuals. Each is observed.

\textbf{Emergent isotropy.} On isotropic 2D PDEs (Darcy, Navier-Stokes), learned $\beta_x$ and $\beta_y$ converge to within $\Delta < 0.035$ of each other without isotropy being prescribed (Appendix Fig.~\ref{fig:isotropy}). On 3D compressible Navier-Stokes, axis-aligned channels settle near $\beta \approx 2$ (Laplacian) while the cross-derivative channel collapses to near-identity (Appendix Fig.~\ref{fig:cns_3d}): the network identifies which directional derivatives matter and matches their order to the observed physics.

\textbf{Resolution-dependent gains.} Empirically, $\partial$-FNO outperforms FNO at every resolution from $N{=}64$ to $N{=}2048$, with the noisy-data improvement decaying from $39\%$ to $14\%$ (Appendix Fig.~\ref{fig:resolution}). The advantage of input-side derivative features grows at coarser resolution, where the network's Nyquist headroom is most constrained.

\textbf{Complementarity with physics-informed training.} PINO \citep{li2024physics} encodes the PDE on the loss side; $\partial$-NO encodes it on the input side. Adding $\partial$-features improves PINO in all $8/8$ tested settings (Appendix Table~\ref{tab:pino}): the two channels do not duplicate but complement each other.
\end{plainblock}

\section{Related Work}
\label{sec:related}

\textbf{Neural operators.} The field has grown rapidly since FNO \citep{li2020fourier} and DeepONet \citep{lu2021learning}. Recent work explores attention mechanisms \citep{cao2021choose, hao2023gnot, wu2024Transolver}, wavelets \citep{tripura2023wavelet}, and implicit representations \citep{serrano2023operator}. Theoretical foundations include approximation bounds \citep{kovachki2023neural} and universal approximation results. Our $\partial$-augmentation is orthogonal to these advances: it wraps any backbone with learnable derivative features, requiring no architectural modification.

\textbf{Local-feature operators.} A parallel line of work injects locality into neural operators \citep{chauhan2026efficient}: Conv-FNO \citep{liu2025enhancing} attaches local convolutions to FNO, LocalNO \citep{liu2024neural} introduces localized integral and differential kernels, and CNO \citep{raonic2023convolutional} replaces spectral layers with structured convolutions throughout. $\partial$-NO is complementary: it provides a specific local operation (fractional differentiation) with a built-in inductive bias for PDE structure. Our $\partial$-CNO experiments (\S\ref{sec:experiments}) directly isolate this bias from generic local processing.

\textbf{Derivative and physics-aware features.} Recent work has explored providing derivative information to neural operators. Equation-aware emulators \citep{zhu2025generalizing} encode full PDE structure including nonlinear terms, achieving improvements but requiring the PDE form a priori. Physics-informed approaches like PINNs \citep{raissi2019physics, thodi2022learning} and the physics-informed neural operator \citep{li2024physics} incorporate derivatives through loss terms via automatic differentiation. Our approach differs in three ways: (1) we learn derivative orders rather than specifying them from known physics; (2) we provide derivatives as input features, computed via FFT or GL; and (3) our bias-variance analysis reveals that optimal orders should be less than the PDE order, a phenomenon that equation-matching approaches do not address.

\textbf{Spectral perspectives.} FNO performs convolutions in Fourier space with a hard mode truncation at $k_{\max}$. Our analysis exploits this structure: the capacity cutoff $\xi_c$ in Theorem~\ref{thm:optimal_beta} is sharply defined for FNO and other neural operators with a fixed number of Fourier modes, making the bias-variance tradeoff particularly sharp. More broadly, spectral bias in neural networks, where low frequencies are learned before high frequencies \citep{rahaman2019spectral}, suggests that derivative features may help by preprocessing high-frequency content into a form the network can access earlier in training.

\textbf{Fractional calculus.} Fractional derivatives model anomalous diffusion, memory effects, and non-local phenomena \citep{metzler2000random}, and have found applications in financial machine learning for constructing stationary features that preserve memory \citep{lopezdeprado2018advances}. We use them differently: as a continuous parameterization enabling gradient-based optimization over derivative order. This treats the fractional derivative as a learnable feature transformation, not as a physical model of sub-diffusion or long-range dependence. Connections between learned orders and anomalous diffusion remain an interesting direction for future work.

\section{Discussion}
\label{sec:discussion}

We introduced $\partial$-NO, a simple augmentation providing learnable fractional derivative features to any neural operator backbone. Through Picard iteration on mild solutions, we showed how derivatives of the initial condition directly contribute to the solution operator, providing theoretical grounding for derivative features. Providing these features improves approximation bounds from $O(W^{-(s-m)/d})$ to $O(W^{-s/d})$, a substantial gain in parameter efficiency.

Our central finding is that the optimal derivative order $\beta^*$ is less than the PDE order $m$, with the gap characterized by a spectral bias-variance tradeoff. The closed-form expression for $\delta^* = m - \beta^*$ correctly predicts how the optimal order responds to noise, sample size, and capacity, as confirmed experimentally. For FNO and other operators with fixed Fourier modes, the capacity cutoff is sharply defined. Learning $\beta$ achieves automatic spectral regularization: the network selects $\beta^*$ in response to noise, sample size, and capacity, as predicted by Theorem~\ref{thm:optimal_beta} and confirmed empirically.

\textbf{Limitations.} Four scope boundaries: (i) the Gr\"unwald-Letnikov stencil assumes a regular grid (extensions via graph fractional Laplacians or RBF operators are open); (ii) Theorem~\ref{thm:optimal_beta}'s sharp predictions require $s > m + d/2$ and degrade to a directional guide near shocks, as the near-shock Burgers experiments empirically confirm; (iii) we preprocess only spatial inputs, leaving temporal fractional derivatives as an orthogonal direction; and (iv) the theorem characterizes the population-optimal $\beta^*$ rather than the optimizer's converged value. Empirically, $\beta$ stabilizes within 200--400 AdamW epochs across every backbone and PDE we tested, but formal trajectory guarantees remain open \citep{mehouachi2025catastrophic, mehouachi2021exact, mehouachi2023sifrian}.

Stepping back from these scope considerations, the central finding $\beta^* < m$ illustrates a broader principle for physics-informed machine learning at finite scales: inductive biases derived from governing equations should be \emph{parameterized rather than prescribed}, allowing the optimizer to locate its own compromise between approximation power and estimation variance. Under finite, noisy data, the statistically optimal physics is rarely the literal equation; learning the inductive bias is what turns physical knowledge into statistical efficiency.

\newpage

\section*{Impact Statement}

This work advances methodology for scientific machine learning, with potential applications in a wide range of scientific fields. The finding that learned derivative orders should be sub-physical has implications for how we incorporate domain knowledge into machine learning systems: matching physics exactly may not always be optimal. We identify no specific ethical concerns.

\section*{Code Availability}
The code for $\partial$-NO is available at \url{https://github.com/FaresBMehouachi/delNO}. The authors declare no competing interests.

\section*{Acknowledgments}
This work was supported by the NYUAD Center for Interacting Urban Networks (CITIES), funded by Tamkeen under the NYUAD Research Institute Award CG001. The views expressed in this article are those of the authors and do not reflect the opinions of CITIES or their funding agencies.

\newpage
\bibliographystyle{unsrtnat}
\bibliography{refs_manual}

@inproceedings{li2020fourier,
  title={Fourier neural operator for parametric partial differential equations},
  author={Li, Zongyi and Kovachki, Nikola and Azizzadenesheli, Kamyar and Liu, Burigede and Bhattacharya, Kaushik and Stuart, Andrew and Anandkumar, Anima},
  booktitle={International Conference on Learning Representations},
  year={2021},
  note={arXiv:2010.08895}
}

@article{lu2021learning,
  title={Learning nonlinear operators via {DeepONet} based on the universal approximation theorem of operators},
  author={Lu, Lu and Jin, Pengzhan and Pang, Guofei and Zhang, Zhongqiang and Karniadakis, George Em},
  journal={Nature machine intelligence},
  volume={3},
  number={3},
  pages={218--229},
  year={2021},
  publisher={Nature Publishing Group UK London}
}

@article{kovachki2023neural,
  title={Neural operator: Learning maps between function spaces with applications to {PDE}s},
  author={Kovachki, Nikola and Li, Zongyi and Liu, Burigede and Azizzadenesheli, Kamyar and Bhattacharya, Kaushik and Stuart, Andrew and Anandkumar, Anima},
  journal={Journal of Machine Learning Research},
  volume={24},
  number={89},
  pages={1--97},
  year={2023}
}

@article{cao2021choose,
  title={Choose a transformer: Fourier or {G}alerkin},
  author={Cao, Shuhao},
  journal={Advances in Neural Information Processing Systems},
  volume={34},
  pages={24924--24940},
  year={2021}
}

@inproceedings{hao2023gnot,
  title={{GNOT}: A general neural operator transformer for operator learning},
  author={Hao, Zhongkai and Wang, Zhengyi and Su, Hang and Ying, Chengyang and Dong, Yinpeng and Liu, Songming and Cheng, Ze and Song, Jian and Zhu, Jun},
  booktitle={International Conference on Machine Learning},
  pages={12556--12569},
  year={2023},
  organization={PMLR}
}

@inproceedings{wu2023solving,
  title={Solving high-dimensional {PDE}s with latent spectral models},
  author={Wu, Haixu and Hu, Tengge and Luo, Huakun and Wang, Jianmin and Long, Mingsheng},
  booktitle={International Conference on Machine Learning},
  pages={37417--37438},
  year={2023},
  organization={PMLR}
}

@inproceedings{wu2024Transolver,
  title={Transolver: A Fast Transformer Solver for {PDE}s on General Geometries},
  author={Haixu Wu and Huakun Luo and Haowen Wang and Jianmin Wang and Mingsheng Long},
  booktitle={International Conference on Machine Learning},
  pages={53681--53705},
  year={2024},
  organization={PMLR}
}

@inproceedings{luo2025transolver++,
  title={Transolver++: An accurate neural solver for {PDE}s on million-scale geometries},
  author={Luo, Huakun and Wu, Haixu and Zhou, Hang and Xing, Lanxiang and Di, Yichen and Wang, Jianmin and Long, Mingsheng},
  booktitle={International Conference on Machine Learning},
  pages={41432--41449},
  year={2025},
  organization={PMLR}
}

@article{tripura2023wavelet,
  title={Wavelet neural operator for solving parametric partial differential equations in computational mechanics problems},
  author={Tripura, Tapas and Chakraborty, Souvik},
  journal={Computer Methods in Applied Mechanics and Engineering},
  volume={404},
  pages={115783},
  year={2023},
  publisher={Elsevier}
}

@article{raissi2019physics,
  title={Physics-informed neural networks: A deep learning framework for solving forward and inverse problems involving nonlinear partial differential equations},
  author={Raissi, Maziar and Perdikaris, Paris and Karniadakis, George E},
  journal={Journal of Computational Physics},
  volume={378},
  pages={686--707},
  year={2019},
  publisher={Elsevier}
}

@article{zhu2025generalizing,
  title={Generalizing {PDE} Emulation with Equation-Aware Neural Operators},
  author={Zhu, Qian-Ze and Raccuglia, Paul and Brenner, Michael P},
  journal={arXiv preprint arXiv:2511.09729},
  year={2025}
}

@inproceedings{rahaman2019spectral,
  title={On the spectral bias of neural networks},
  author={Rahaman, Nasim and Baratin, Aristide and Arpit, Devansh and Draxler, Felix and Lin, Min and Hamprecht, Fred and Bengio, Yoshua and Courville, Aaron},
  booktitle={International Conference on Machine Learning},
  pages={5301--5310},
  year={2019},
  organization={PMLR}
}

@article{yarotsky2017error,
title = {Error bounds for approximations with deep {ReLU} networks},
journal = {Neural Networks},
volume = {94},
pages = {103-114},
year = {2017},
issn = {0893-6080},
author = {Dmitry Yarotsky}
}

@article{metzler2000random,
  title={The random walk's guide to anomalous diffusion: a fractional dynamics approach},
  author={Metzler, Ralf and Klafter, Joseph},
  journal={Physics reports},
  volume={339},
  number={1},
  pages={1--77},
  year={2000},
  publisher={Elsevier}
}

@article{pang2019fpinns,
  title={{fPINNs}: Fractional physics-informed neural networks},
  author={Pang, Guofei and Lu, Lu and Karniadakis, George Em},
  journal={SIAM Journal on Scientific Computing},
  volume={41},
  number={4},
  pages={A2603--A2626},
  year={2019},
  publisher={SIAM}
}

@article{guo2022monte,
  title={{M}onte {C}arlo {fPINNs}: Deep learning method for forward and inverse problems involving high dimensional fractional partial differential equations},
  author={Guo, Ling and Wu, Hao and Yu, Xiaochen and Zhou, Tao},
  journal={Computer Methods in Applied Mechanics and Engineering},
  volume={400},
  pages={115523},
  year={2022},
  publisher={Elsevier}
}

@article{javadi2023solving,
  title={Solving fractional order differential equations by using fractional radial basis function neural network},
  author={Javadi, Rana and Mesgarani, Hamid and Nikan, Omid and Avazzadeh, Zakieh},
  journal={Symmetry},
  volume={15},
  number={6},
  pages={1275},
  year={2023},
  publisher={MDPI}
}

@article{ren2023class,
  title={A class of improved fractional physics informed neural networks},
  author={Ren, Hongpeng and Meng, Xiangyun and Liu, Rongrong and Hou, Jian and Yu, Yongguang},
  journal={Neurocomputing},
  volume={562},
  pages={126890},
  year={2023},
  publisher={Elsevier}
}

@article{khan2018novel,
  title={A novel fractional gradient-based learning algorithm for recurrent neural networks},
  author={Khan, Shujaat and Ahmad, Jawwad and Naseem, Imran and Moinuddin, Muhammad},
  journal={Circuits, Systems, and Signal Processing},
  volume={37},
  number={2},
  pages={593--612},
  year={2018},
  publisher={Springer}
}

@article{shin2023accelerating,
  title={Accelerating gradient descent and {A}dam via fractional gradients},
  author={Shin, Yeonjong and Darbon, J{\'e}r{\^o}me and Karniadakis, George Em},
  journal={Neural Networks},
  volume={161},
  pages={185--201},
  year={2023},
  publisher={Elsevier}
}

@article{elnady2025comprehensive,
  title={A comprehensive survey of fractional gradient descent methods and their convergence analysis},
  author={Elnady, Sroor M and El-Beltagy, Mohamed and Radwan, Ahmed G and Fouda, Mohammed E},
  journal={Chaos, Solitons \& Fractals},
  volume={194},
  pages={116154},
  year={2025},
  publisher={Elsevier}
}

@inproceedings{cui2025neural,
  title={Neural variable-order fractional differential equation networks},
  author={Cui, Wenjun and Kang, Qiyu and Li, Xuhao and Zhao, Kai and Tay, Wee Peng and Deng, Weihua and Li, Yidong},
  booktitle={Proceedings of the AAAI Conference on Artificial Intelligence},
  volume={39},
  pages={16109--16117},
  year={2025}
}

@article{pu2016fractional,
  author={Pu, Yi-Fei and Yi, Zhang and Zhou, Ji-Liu},
  journal={IEEE Transactions on Neural Networks and Learning Systems}, 
  title={Fractional {H}opfield Neural Networks: Fractional Dynamic Associative Recurrent Neural Networks}, 
  year={2017},
  volume={28},
  number={10},
  pages={2319-2333},
  doi={10.1109/TNNLS.2016.2582512}}

@inproceedings{kang2024unleashing,
  title={Unleashing the potential of fractional calculus in graph neural networks with {FROND}},
  author={Kang, Qiyu and Zhao, Kai and Ding, Qinxu and Ji, Feng and Li, Xuhao and Liang, Wenfei and Song, Yang and Tay, Wee Peng},
  booktitle={International Conference on Learning Representations},
  year={2024},
  note={arXiv:2404.17099}
}

@article{coelho2024neural,
  title={Neural fractional differential equations: Optimising the order of the fractional derivative},
  author={Coelho, Cecilia and Costa, M Fernanda P and Ferras, Luis L},
  journal={{F}ractal and {F}ractional},
  volume={8},
  number={9},
  pages={529},
  year={2024},
  publisher={MDPI}
}

@article{raonic2023convolutional,
  title={Convolutional neural operators for robust and accurate learning of {PDE}s},
  author={Raonic, Bogdan and Molinaro, Roberto and De Ryck, Tim and Rohner, Tobias and Bartolucci, Francesca and Alaifari, Rima and Mishra, Siddhartha and De B{\'e}zenac, Emmanuel},
  journal={{A}dvances in {N}eural {I}nformation {P}rocessing {S}ystems},
  volume={36},
  pages={77187--77200},
  year={2023}
}

@article{li2024physics,
  title={Physics-informed neural operator for learning partial differential equations},
  author={Li, Zongyi and Zheng, Hongkai and Kovachki, Nikola and Jin, David and Chen, Haoxuan and Liu, Burigede and Azizzadenesheli, Kamyar and Anandkumar, Anima},
  journal={ACM/IMS Journal of Data Science},
  volume={1},
  number={3},
  pages={1--27},
  year={2024},
  publisher={ACM New York, NY}
}

@article{mehouachi2026flowmixer,
  title={{FlowMixer}: A Depth-Agnostic Neural Architecture for Interpretable Spatiotemporal Forecasting},
  author={Mehouachi, Fares B. and Jabari, Saif Eddin},
  journal={Advances in Neural Information Processing Systems},
  volume={38},
  pages={88811--88861},
  year={2025}
}

@article{yang2025urban,
  title={Urban traffic analysis and forecasting through shared {K}oopman eigenmodes},
  author={Yang, Chuhan and Mehouachi, Fares B and Menendez, Monica and Jabari, Saif Eddin},
  journal={Nonlinear Dynamics},
  volume={113},
  number={23},
  pages={32307--32328},
  year={2025},
  publisher={Springer}
}

@inproceedings{liu2024neural,
  title={Neural operators with localized integral and differential kernels},
  author={Liu-Schiaffini, Miguel and Berner, Julius and Bonev, Boris and Kurth, Thorsten and Azizzadenesheli, Kamyar and Anandkumar, Anima},
  booktitle={International Conference on Machine Learning},
  pages={32576--32594},
  year={2024},
  organization={PMLR}
}

@article{chauhan2026efficient,
  title={Efficient Dilated Squeeze and Excitation Neural Operator for Differential Equations},
  author={Chauhan, Prajwal and Choutri, Salah Eddine and Jabari, Saif Eddin},
  journal={Transactions on Machine Learning Research},
  year={2026},
  note={arXiv:2601.17407}
}

@article{liu2025enhancing,
  title={Enhancing {F}ourier neural operators with local spatial features},
  author={Liu, Chaoyu and Murari, Davide and Liu, Lihao and Li, Yangming and Budd, Chris and Sch{\"o}nlieb, Carola-Bibiane},
  journal={arXiv preprint arXiv:2503.17797},
  year={2025}
}

@inproceedings{serrano2023operator,
 author = {Serrano, Louis and Le Boudec, Lise and Kassa\"{\i} Koupa\"{\i}, Armand and Wang, Thomas X and Yin, Yuan and Vittaut, Jean-No\"{e}l and Gallinari, Patrick},
 booktitle = {Advances in Neural Information Processing Systems},
 pages = {70581--70611},
 title = {Operator Learning with Neural Fields: Tackling {PDE}s on General Geometries},
 volume = {36},
 year = {2023}
}

@book{lopezdeprado2018advances,
  author    = {{L{\'o}pez de Prado}, Marcos M.},
  title     = {Advances in Financial Machine Learning},
  publisher = {John Wiley \& Sons},
  address   = {Hoboken, NJ},
  year      = {2018},
  edition   = {1st},
  isbn      = {978-1-119-48208-6}
}

@book{stein1999interpolation,
  author    = {Stein, Michael L.},
  title     = {Interpolation of Spatial Data: Some Theory for Kriging},
  series    = {Springer Series in Statistics},
  publisher = {Springer},
  address   = {New York, NY},
  year      = {1999},
  isbn      = {978-0-387-98629-6},
  doi       = {10.1007/978-1-4612-1494-6}
}

@book{carl1990entropy,
  author    = {Carl, Bernd and Stephani, Irmtraud},
  title     = {Entropy, Compactness and the Approximation of Operators},
  series    = {Cambridge Tracts in Mathematics},
  volume    = {98},
  publisher = {Cambridge University Press},
  address   = {Cambridge},
  year      = {1990}
}

@misc{thuml2024neuraloperator,
  author       = {{THUML}},
  title        = {{N}eural-{S}olver-{L}ibrary: A Library for Advanced Neural {PDE} Solvers},
  year         = {2025},
  howpublished = {\url{https://github.com/thuml/Neural-Solver-Library}},
  note         = {Tsinghua Machine Learning Group}
}

@book{podlubny1999fractional,
  title={Fractional differential equations: an introduction to fractional derivatives, fractional differential equations, to methods of their solution and some of their applications},
  author={Podlubny, Igor},
  volume={198},
  year={1999},
  publisher={Academic Press}
}

@article{lanthaler2022error,
  title={Error estimates for {DeepONets}: A deep learning framework in infinite dimensions},
  author={Lanthaler, Samuel and Mishra, Siddhartha and Karniadakis, George E},
  journal={Transactions of Mathematics and its Applications},
  volume={6},
  number={1},
  year={2022},
  publisher={Oxford University Press}
}

@inproceedings{takamoto2022pdebench,
  title={{PDEB}ench: An Extensive Benchmark for Scientific Machine Learning},
  author={Takamoto, Makoto and Praditia, Timothy and Leiteritz, Raphael and MacKinlay, Daniel and Alesiani, Francesco and Pfl{\"u}ger, Dirk and Niepert, Mathias},
  booktitle={Advances in Neural Information Processing Systems Datasets and Benchmarks Track},
  year={2022}
}

@inproceedings{mehouachi2021uav_gass,
  title={Detection and Identification of {UAV}s based on Spectrum Monitoring and Deep Learning in Negative {SNR} Conditions},
  author={Mehouachi, Fares B. and Galvis, Juan and Morales, Santiago and Meriac, Milosch and Vega, Felix and Kasmi, Chaouki},
  booktitle={XXXIV URSI General Assembly and Scientific Symposium (URSI GASS)},
  year={2021}
}

@article{mehouachi2021uav_rsl,
  title={Detection of {UAV}s based on Spectrum Monitoring and Deep Learning in Negative {SNR} Conditions},
  author={Mehouachi, Fares B. and Yang, Qingjie and Galvis, Juan and Morales, Santiago and Meriac, Milosch and Vega, Felix and Kasmi, Chaouki},
  journal={URSI Radio Science Letters},
  volume={3},
  year={2021},
  doi={10.46620/21-0043}
}

@misc{mehouachi2023sifrian,
  title={Systems, methods, and computer-readable media for utilizing a {S}ifrian inversion to build a model to generate an image of a surveyed medium},
  author={Mehouachi, Fares B. and Yang, Qingjie and Kasmi, Chaouki},
  year={2023},
  month={jan},
  note={{US} Patent 11,551,416},
  howpublished={{US} Patent 11,551,416}
}

@article{mehouachi2021exact,
  title={Exact Stochastic Second Order Deep Learning},
  author={Mehouachi, Fares B. and Kasmi, Chaouki},
  journal={arXiv preprint arXiv:2104.03804},
  year={2021}
}

@article{mehouachi2025catastrophic,
  title={Catastrophic Overfitting, Entropy Gap, and Participation Ratio: A Noiseless $\ell^p$ Norm Solution for Fast Adversarial Training},
  author={Mehouachi, Fares B. and Jabari, Saif Eddin},
  journal={arXiv preprint arXiv:2505.02360},
  year={2025}
}

@article{thodi2024fourier,
  title={Fourier neural operator for learning solutions to macroscopic traffic flow models: Application to the forward and inverse problems},
  author={Thodi, Bilal Thonnam and Ambadipudi, Sai Venkata Ramana and Jabari, Saif Eddin},
  journal={Transportation Research Part C: Emerging Technologies},
  volume={160},
  pages={104500},
  year={2024},
  publisher={Elsevier}
}

@inproceedings{thodi2022learning,
  title={Learning-based solutions to nonlinear hyperbolic {PDE}s: Empirical insights on generalization errors},
  author={Thodi, Bilal Thonnam and Ambadipudi, Sai Venkata Ramana and Jabari, Saif Eddin},
  booktitle={Fifth Workshop on Machine Learning and the Physical Sciences (NeurIPS 2022)},
  year={2022}
}

\newpage
\newpage
\appendix

\section{Table of Notation}
\label{app:notation}

\begin{table}[h!]
\centering
\caption{Summary of notation used throughout the paper.}
\label{tab:notation}
\begin{tabular}{@{}ll@{}}
\toprule
\textbf{Notation} & \textbf{Meaning} \\
\midrule
\multicolumn{2}{l}{\textit{Spaces and functions}} \\
$\Omega \subset \R^d$ & Spatial domain \\
$d$ & Spatial dimension \\
$H^s(\Omega)$ & Sobolev space of regularity $s$ \\
$B_s(R)$ & Sobolev ball $\{u \in H^s : \|u\|_{H^s} \leq R\}$ \\
$u, u_0$ & Input function / initial condition \\
$v$ & Solution function \\
\midrule
\multicolumn{2}{l}{\textit{PDE structure}} \\
$m$ & Differential order of the PDE \\
$\Lop$ & Linear differential operator \\
$\Non$ & Nonlinearity \\
$N$ & Degree of polynomial nonlinearity \\
$G_t$ & Green's function / semigroup $e^{t\Lop}$ \\
$\nu$ & Viscosity coefficient \\
\midrule
\multicolumn{2}{l}{\textit{Operators and features}} \\
$\G, \G^\dagger$ & Solution operator (true / ground truth) \\
$\G_\theta$ & Neural operator with parameters $\theta$ \\
$\Phi^{[m]}$ & Feature map $(u, D^1 u, \ldots, D^m u)$ \\
$D^\beta$ & Fractional derivative of order $\beta$ \\
$\Dfrak^\beta$ & GL approximation (Alternatively Weyl via FFT on periodic domains) \\
\midrule
\multicolumn{2}{l}{\textit{Derivative orders}} \\
$\beta$ & Learnable derivative order \\
$\beta^*$ & Optimal derivative order (MSE-minimizing) \\
$\delta^* = m - \beta^*$ & Gap from PDE order \\
$s_k$ & Learnable scale for $k$-th derivative feature \\
\midrule
\multicolumn{2}{l}{\textit{Spectral quantities}} \\
$\xi$ & Frequency variable \\
$\hat{u}(\xi) = \mathcal{F}[u]$ & Fourier transform of $u$ \\
$S(\xi)$ & Power spectral density $\E[|\hat{u}(\xi)|^2]$ \\
$\xi_c$ & Network capacity cutoff frequency \\
$\xi_{\max} = \pi/h$ & Nyquist frequency \\
$k, k_{\max}$ & Fourier modes / maximum modes \\
\midrule
\multicolumn{2}{l}{\textit{Statistical quantities}} \\
$n$ & Number of training samples \\
$\sigma^2$ & Observation noise variance \\
$\MSE, B, V$ & Mean squared error, bias, variance \\
$\tau = 2(s-m)$ & Regularity gap parameter \\
\midrule
\multicolumn{2}{l}{\textit{Implementation}} \\
$W$ & Network width / number of parameters \\
$h$ & Grid spacing \\
$K$ / $K_{\mathrm{GL}}$ & Number of derivative channels / GL stencil length \\
$w_k^{(\beta)}$ & GL weights for order $\beta$ \\
$\psi = \Gamma'/\Gamma$ & Digamma function \\
\bottomrule
\end{tabular}
\end{table}

\newpage

\section{Proof of Theorem~\ref{thm:derivative_emergence} (Derivative Emergence)}
\label{app:picard}

We prove that derivatives of the initial condition naturally enter the solution operator through Picard iteration.

\begin{proof}
Consider the semilinear PDE:
\begin{equation}
\partial_t u = \Lop u + \Non[u], \quad u(0) = u_0,
\end{equation}
where $\Lop$ is a linear differential operator of order $m$ and $\Non$ is a polynomial nonlinearity of degree $N$.

\textbf{Step 1: Mild solution formulation.}
The mild solution satisfies the integral equation:
\begin{equation}
u(t) = G_t \ast u_0 + \int_0^t G_{t-s} \ast \Non[u(s)] \, ds,
\end{equation}
where $G_t = e^{t\Lop}$ is the semigroup generated by $\Lop$. For the heat operator $\Lop = \nu \Delta$, this is the heat kernel. For general $\Lop$, it is the Green's function.

\textbf{Step 2: Picard iteration.}
Define iterates:
\begin{align}
u^{(0)}(t) &= G_t \ast u_0, \\
u^{(k+1)}(t) &= G_t \ast u_0 + \int_0^t G_{t-s} \ast \Non[u^{(k)}(s)] \, ds.
\end{align}
Under standard assumptions (bounded $u_0$, Lipschitz $\Non$, finite time horizon), the Banach fixed-point theorem guarantees convergence $u^{(k)} \to u$ in an appropriate function space.

\textbf{Step 3: Expansion of the semigroup.}
The semigroup satisfies $\partial_t G_t = \Lop G_t$ with $G_0 = I$. For analytic semigroups (which include all parabolic operators with constant coefficients), Taylor expanding in $s$:
\begin{equation}
G_s \ast u_0 = u_0 + s \cdot \Lop u_0 + \frac{s^2}{2} \cdot \Lop^2 u_0 + O(s^3).
\end{equation}
The term $\Lop u_0$ involves $m$-th order derivatives of $u_0$ (e.g., $\Lop u_0 = \nu \Delta u_0$ for diffusion). The $O(s^3)$ remainder is uniform over $u_0 \in B_s(R)$ by standard semigroup estimates.

\textbf{Step 4: Expansion of the nonlinearity.}
Since $\Non$ is Fr\'echet differentiable:
\begin{equation}
\Non[G_s \ast u_0] = \Non[u_0 + s \Lop u_0 + O(s^2)] = \Non[u_0] + s \cdot D\Non[u_0][\Lop u_0] + O(s^2),
\end{equation}
where $D\Non[u_0][v]$ is the directional derivative of $\Non$ at $u_0$ in direction $v$.

For polynomial $\Non$, this derivative has explicit form. For example, if $\Non[u] = -u \partial_x u$, then:
\begin{equation}
D\Non[u_0][v] = -v \partial_x u_0 - u_0 \partial_x v.
\end{equation}
When $v = \Lop u_0 = \nu \partial_{xx} u_0$, we get terms involving $\partial_{xx} u_0 \cdot \partial_x u_0$ and $u_0 \cdot \partial_{xxx} u_0$.

\textbf{Step 5: First Picard iterate.}
The first iterate at final time $T$ is:
\begin{equation}
u^{(1)}(T) = G_T \ast u_0 + \int_0^T G_{T-s} \ast \Non[G_s \ast u_0] \, ds.
\end{equation}
Substituting the expansions:
\begin{align}
\int_0^T G_{T-s} \ast \Non[G_s \ast u_0] \, ds &= \int_0^T G_{T-s} \ast \left( \Non[u_0] + s \cdot D\Non[u_0][\Lop u_0] + O(s^2) \right) ds \\
&= \Non[u_0] \ast \int_0^T G_{T-s} \, ds + D\Non[u_0][\Lop u_0] \ast \int_0^T s \cdot G_{T-s} \, ds + O(T^3),
\end{align}
where we exchanged the order of convolution and integration by Fubini's theorem (valid since $G_{T-s}$ is integrable on $[0,T]$ for the semigroups considered).

\textbf{Step 6: Averaged Green's operators.}
Define:
\begin{align}
\bar{G}_T &:= \frac{1}{T} \int_0^T G_{T-s} \, ds, \\
\tilde{G}_T &:= \frac{2}{T^2} \int_0^T s \cdot G_{T-s} \, ds.
\end{align}
These are bounded linear operators for $T > 0$ (the semigroup is smoothing).

\textbf{Step 7: Structure of the solution operator.}
Combining:
\begin{equation}
u^{(1)}(T) = G_T \ast u_0 + T \cdot \bar{G}_T \ast \Non[u_0] + \frac{T^2}{2} \cdot \tilde{G}_T \ast D\Non[u_0][\Lop u_0] + O(T^3).
\end{equation}

The true solution $u(T) = \lim_{k \to \infty} u^{(k)}(T)$ differs from $u^{(1)}(T)$ by higher-order Picard corrections. For fixed $T < T_*(R)$, these corrections are bounded perturbations that preserve the dependence structure. Defining the remainder $\mathcal{R}(u_0, T) := \G[u_0] - G_T \ast u_0 - T \cdot \bar{G}_T \ast \Non[u_0]$, we have $\|\mathcal{R}(u_0, T)\|_{L^2} \leq C(R)T^2$ by standard Picard remainder estimates. The constant $C(R)$ depends polynomially on $R$ through the Lipschitz constant of $\Non$ on $B_s(R)$ and the semigroup bounds.

\textbf{Step 8: Identifying the features.}
The solution depends on $u_0$ through:
\begin{itemize}
\item $u_0$ (zeroth order): appears in $G_T \ast u_0$ and in $\Non[u_0]$
\item $\Lop u_0$ (order $m$): appears in the expansion of $G_s \ast u_0$
\item $\partial^j u_0$ for $j \leq m$: appear in $\Non[u_0]$ if $\Non$ involves derivatives
\item $u_0^n$ for $n \leq N$: appear in $\Non[u_0]$ if $\Non$ is polynomial degree $N$
\end{itemize}

This completes the proof that the solution operator $\G[u_0] = u(T)$ depends on derivatives of $u_0$ up to order $m$ and powers up to $u_0^N$.
\end{proof}

\subsection{Explicit Features for Specific PDEs}

\paragraph{Burgers equation.}
$\partial_t u + u \partial_x u = \nu \partial_{xx} u$, so $\Lop = \nu \partial_{xx}$ and $\Non[u] = -u \partial_x u$.

Features: $\{u_0, \partial_x u_0, \partial_{xx} u_0, u_0 \partial_x u_0\}$.

\paragraph{Navier-Stokes (2D vorticity).}
$\partial_t \omega + \mathbf{u} \cdot \nabla \omega = \nu \Delta \omega$, where $\mathbf{u} = \nabla^\perp \psi$ and $\Delta \psi = -\omega$.

Features: $\{\omega_0, \nabla \omega_0, \Delta \omega_0, \mathbf{u}_0 \cdot \nabla \omega_0\}$.

\paragraph{Darcy flow.}
$-\nabla \cdot (a \nabla p) = f$. This is elliptic; the solution operator maps coefficient $a \mapsto$ pressure $p$.

Features: $\{a, \nabla a, \nabla^2 a\}$ (encoding coefficient variations).

\paragraph{KdV.} $\partial_t u + u\,\partial_x u + \delta\,\partial_{xxx} u = 0$ with $\delta = 2.5\times10^{-3}$ (App.~\ref{app:experiments}), so $\Lop = -\delta\,\partial_{xxx}$ ($m=3$) and $\Non[u] = -u\,\partial_x u$. The dependence on $\partial_{xxx} u_0$, and hence the feature set $\{u_0,\,\partial_x u_0,\,\partial_{xx} u_0,\,\partial_{xxx} u_0,\,u_0\,\partial_x u_0\}$, is independent of the value of $\delta$.

\subsection{Extension to Arbitrary Times and Stationary Problems}
\label{app:picard_extended}

Theorem~\ref{thm:derivative_emergence} establishes derivative dependence for $T < T_*(R)$. We now show this structure persists for arbitrary times and extends to stationary problems.

\paragraph{Iterated Picard for large $T$.}
The key observation is simple: once derivatives $\{D^j u_0\}_{j \leq m}$ appear in the first Picard iterate, they propagate through all subsequent iterations. For $T > T_*(R)$, the solution $u(T)$ can be written as a (possibly complex) functional of $\{u_0, D u_0, \ldots, D^m u_0\}$. The specific form of this functional (involving compositions of semigroups, integrals, and nonlinear operations) is precisely what the neural operator learns. Our contribution is showing that providing $\{D^j u_0\}$ as input features gives the network direct access to the quantities it needs, rather than requiring it to learn differentiation implicitly.

\paragraph{Stationary (elliptic) problems.}
For stationary PDEs like Darcy flow $-\nabla \cdot (a \nabla p) = f$, there is no time evolution. The solution operator $\G: a \mapsto p$ depends directly on derivatives of the coefficient through the elliptic structure.

Expanding around a reference coefficient $a_0$:
\begin{equation}
p[a] = p[a_0] + \mathcal{L}^{-1}[\nabla \cdot (\delta a \nabla p[a_0])] + O(\|\delta a\|^2),
\end{equation}
where $\mathcal{L} = -\nabla \cdot (a_0 \nabla \cdot)$ and $\delta a = a - a_0$. The linear response involves $\nabla(\delta a)$, so the solution operator depends on $\{a, \nabla a\}$. Higher-order terms involve repeated first derivatives of $\delta a$ rather than higher derivatives. We nevertheless supply $\nabla^2 a$ as a redundant feature, since the elliptic operator is second order, yielding the feature set $\{a, \nabla a, \nabla^2 a\}$ for this second-order operator.

The key insight is identical to the parabolic case: the PDE structure induces dependence on derivatives of the input, and providing these as features aids learning.

\paragraph{Quasilinear and non-polynomial nonlinearities.}
The analysis extends beyond strictly polynomial $\Non$. For quasilinear PDEs like Navier-Stokes, where the velocity $\mathbf{u} = \nabla^\perp \Delta^{-1} \omega$ involves inverting the Laplacian, the key observation is that $\Delta^{-1}$ is a \emph{bounded} operator on Sobolev spaces; it does not introduce new dependence on derivatives of $u_0$. What matters is that the PDE structure (through $\Lop$ and $\Non$) creates dependence on $\{u_0, D^j u_0\}_{j \leq m}$; other operations (like $\Delta^{-1}$) are learned by the neural operator. This justifies applying our framework to Navier-Stokes (with vorticity input $\omega$) as a second-order problem.

\subsection{General Evolution Equations: A Spectral Argument}
\label{app:spectral_general}
\begin{plainblock}

The Picard analysis of \S\ref{app:picard} establishes derivative dependence under the parabolic hypotheses of Theorem~\ref{thm:derivative_emergence} (semilinear evolution with semigroup-generating $\Lop$ and polynomial $\Non$). Here we present a complementary spectral argument that extends the same derivative structure to general evolution equations, including hyperbolic and fractional-in-time PDEs. This addresses the role of the temporal order in determining the derivative-feature structure.

\paragraph{Setup.} Consider the abstract evolution problem
\begin{equation}
\label{eq:abstract_evolution}
\partial_t^a u = \Lop u, \quad u(\cdot, 0) = u_0, \quad t \in [0, T],
\end{equation}
where $\partial_t^a$ is the Caputo fractional derivative of order $a > 0$ and $\Lop$ is a linear operator whose Fourier symbol $P(\xi)$ satisfies $P(\xi) \asymp |\xi|^m$ as $|\xi| \to \infty$. For $a > 1$, the equation requires $\lceil a \rceil$ initial conditions; we set the higher-order ones to zero for clarity (they enter linearly and admit the same analysis). Assume zero-mean initial data $\hat{u}_0(0) = 0$, standard for periodic domains and PDEs with vanishing constants.

\paragraph{Step 1: Fourier reduction.} The spatial Fourier transform reduces \eqref{eq:abstract_evolution} to a family of scalar fractional ODEs indexed by $\xi$:
\begin{equation}
\partial_t^a \hat{u}(t, \xi) = P(\xi)\, \hat{u}(t, \xi), \quad \hat{u}(0, \xi) = \hat{u}_0(\xi).
\end{equation}

\paragraph{Step 2: Laplace transform and resolvent.} Applying the Laplace transform in $t$ and solving algebraically yields
\begin{equation}
\hat{u}(T, \xi) = R_a(P(\xi), T) \cdot \hat{u}_0(\xi),
\end{equation}
where the resolvent multiplier $R_a(z, T)$ is the inverse Laplace transform of $s^{a-1}/(s^a - z)$. Three canonical cases:
\begin{itemize}[leftmargin=2em,itemsep=0pt]
\item Parabolic ($a = 1$): $R_1(z, T) = e^{zT}$ (heat-type semigroup).
\item Hyperbolic ($a = 2$): $R_2(z, T) = \cos(\sqrt{-z}\, T)$ (wave-type oscillator).
\item Fractional ($0 < a < 2,\ a \ne 1$): $R_a(z, T) = E_a(z T^a)$, with $E_a$ the Mittag-Leffler function.
\end{itemize}

\paragraph{Step 3: Factorization through $|\xi|^{-m}$.} The symbol satisfies $P(\xi) = \Theta(|\xi|^m)$, while $R_a(P(\xi),T)$ is bounded uniformly in $\xi$ in each of the three cases above (the standard well-posedness condition $\operatorname{Re} P(\xi) \le 0$). The zero-mean condition $\hat{u}_0(0) = 0$ ensures $|\xi| > 0$ on the support of $\hat{u}_0$, so we may factor
\begin{equation}
R_a(P(\xi), T)\, \hat{u}_0(\xi)
= \underbrace{\frac{R_a(P(\xi), T)}{|\xi|^m}}_{=: \, Q_a(\xi, T)} \cdot \underbrace{|\xi|^m \hat{u}_0(\xi)}_{= \, \widehat{D^m u_0}(\xi)}.
\end{equation}
The first factor $Q_a(\xi, T)$ is a Fourier multiplier of order $-m$, bounded uniformly in $\xi$ for fixed $T$ on the (discrete, zero-mean) periodic spectrum. The second factor is the Fourier transform of $D^m u_0$. Inverting the transform gives
\begin{equation}
\G[u_0] \;=\; u(\cdot, T) \;=\; Q_a(T) \ast D^m u_0,
\end{equation}
where $Q_a(T)$ is a bounded linear operator on $L^2(\Omega)$.

\paragraph{Step 4: Consequence for derivative features.} Providing $D^b u_0$ as an input feature for any $b \le m$ reduces the residual differentiation burden from order $m$ to order $m - b$:
\begin{equation}
\G[u_0] \;=\; Q_a(T) \ast D^{m-b}(D^b u_0).
\end{equation}
This is the spectral analogue of the Picard derivative structure (Theorem~\ref{thm:derivative_emergence}). It holds for every temporal order $a > 0$ for which $R_a$ is well-defined, with the caveat that the operator norm $\|Q_a(T)\|$ depends on $a$ and $T$ but \emph{not} on the choice of input order $b$.

\paragraph{Nonlinear extension.} For semilinear PDEs $\partial_t^a u = \Lop u + \Non[u]$, linearizing around the linear flow $u^{\mathrm{lin}}(t) = R_a(\Lop, t) u_0$ and applying Duhamel's principle yields
\begin{equation}
u(T) = R_a(\Lop, T) u_0 + \int_0^T R_a(\Lop, T - s)\, \Non[u^{\mathrm{lin}}(s)]\, ds + \mathcal{R}(u_0, T),
\end{equation}
with $\|\mathcal{R}\|_{L^2} \le C(R) T^{2a}$ uniformly on $B_s(R)$ under Assumption~\ref{ass:regularity} and the hypotheses of Theorem~\ref{thm:derivative_emergence}. The Fr\'echet derivative of $\G$ satisfies a linearized equation with the same resolvent structure, so the factorization through $|\xi|^{-m}$ persists term by term.

\paragraph{Role of the temporal order.} The temporal order $a$ enters through three quantities: (i) the number of initial conditions $\lceil a \rceil$ (information at $t = 0$); (ii) the specific functional form of $R_a$ (heat for $a = 1$, wave for $a = 2$, Mittag-Leffler for $0 < a < 2$); and (iii) the rate at which $\|Q_a(T)\|$ decays in $T$. Crucially, $a$ does not change the derivative-feature structure, which is determined entirely by the spatial order $m$ via $P(\xi) \asymp |\xi|^m$. This is why our framework, calibrated on parabolic ($a = 1$) Burgers and elliptic Darcy, transfers to other temporal regimes provided the spatial order $m$ is matched.
\end{plainblock}
\section{Proof of Proposition~\ref{prop:stability}: Bi-Lipschitz Stability}
\label{app:stability}

\begin{proof}
Let $\Phi^{[m]}[u] = (u, D^1 u, \ldots, D^m u)$ map $H^s(\Omega) \to \mathcal{X} = H^s \times H^{s-1} \times \cdots \times H^{s-m}$.

\textbf{Lower bound (information preservation).}
Since $\Phi^{[m]}$ includes $u$ as its first component:
\begin{equation}
\|\Phi^{[m]}[u_1] - \Phi^{[m]}[u_2]\|_{\mathcal{X}}^2 = \|u_1 - u_2\|_{H^s}^2 + \sum_{j=1}^m \|D^j(u_1 - u_2)\|_{H^{s-j}}^2 \geq \|u_1 - u_2\|_{H^s}^2.
\end{equation}
Taking square roots gives the lower bound with constant 1.

\textbf{Upper bound (Lipschitz continuity).}
By the definition of Sobolev norms:
\begin{equation}
\|D^j u\|_{H^{s-j}}^2 = \int (1 + |\xi|^2)^{s-j} |\xi|^{2j} |\hat{u}(\xi)|^2 \, d\xi \leq \int (1 + |\xi|^2)^s |\hat{u}(\xi)|^2 \, d\xi = \|u\|_{H^s}^2,
\end{equation}
since $(1 + |\xi|^2)^{s-j} |\xi|^{2j} \leq (1 + |\xi|^2)^s$ for $j \leq m \leq s$.

Therefore:
\begin{equation}
\|\Phi^{[m]}[u_1] - \Phi^{[m]}[u_2]\|_{\mathcal{X}}^2 \leq (1 + m) \|u_1 - u_2\|_{H^s}^2.
\end{equation}
Taking square roots: $\|\Phi^{[m]}[u_1] - \Phi^{[m]}[u_2]\|_{\mathcal{X}} \leq \sqrt{1+m} \|u_1 - u_2\|_{H^s}$.
\end{proof}

\section{Proof of Theorem~\ref{thm:approx}: Approximation Rates}
\label{app:approx}

This section provides the approximation-theoretic analysis establishing that derivative features admit improved approximation bounds. This complements the bias-variance analysis in Section~\ref{sec:optimal_beta} and Appendix~\ref{app:optimal_beta}: the approximation theory shows \emph{that} features help (the established bound improves from $W^{-(s-m)/d}$ to $W^{-s/d}$), while the bias-variance analysis determines \emph{which} derivative order is optimal ($\beta^* < m$). The two frameworks answer different questions and use different tools; their qualitative agreement ($\beta^* < m$ is beneficial) provides evidence of robustness.

\begin{proof}
\textbf{Part (i): Upper bound without features.}

The operator $\G$ depends on derivatives up to order $m$. By Sobolev interpolation, for any $f \in H^s$ with $\|f\|_{H^s} \leq R$:
\begin{equation}
\|f\|_{H^m} \leq C \|f\|_{L^2}^{(s-m)/s} \|f\|_{H^s}^{m/s} \leq C R^{m/s} \|f\|_{L^2}^{(s-m)/s}.
\end{equation}
This interpolation inequality determines the modulus of continuity of $\G$. Since $\G$ depends on $D^m u$ through a Lipschitz function on bounded sets (Theorem~\ref{thm:derivative_emergence}), the operator inherits H\"older continuity with exponent $(s-m)/s$:
\begin{equation}
\|\G[u_1] - \G[u_2]\|_{L^2} \lesssim \|u_1 - u_2\|_{L^2}^{(s-m)/s}.
\end{equation}

The $\epsilon$-covering number of $B_s(R)$ in $L^2$ satisfies $\log \mathcal{N}(\epsilon) \asymp (R/\epsilon)^{d/s}$. By the entropy lower bound of \citet{carl1990entropy}, approximating H\"older-$\gamma$ functions on sets with entropy exponent $\rho = d/s$ requires:
\begin{equation}
\text{error} \geq c \cdot W^{-\gamma s / d} = c \cdot W^{-(s-m)/d}.
\end{equation}
This is a worst-case statement over that function class, not a lower bound for $\G$ itself. Theorem~\ref{thm:approx} uses only the achievable side (Remark~\ref{rem:scope}). The achievable rate $W^{-(s-m)/d}$ follows from \citet{yarotsky2017error}.

\textbf{Part (ii): Upper bound with features.}

Define $\tilde{\G}: \mathcal{X} \to L^2$ by $\tilde{\G}[\Phi^{[m]}[u]] := \G[u]$.

\textit{Step 1: $\tilde{\G}$ is Lipschitz on bounded balls.}

\begin{lemma}[Lipschitz Property of Lifted Operator]
\label{lem:lipschitz_lifted}
On the image $\Phi^{[m]}[B_s(R)]$, the lifted operator $\tilde{\G}$ is Lipschitz:
\begin{equation}
\|\tilde{\G}(\Phi^{[m]}[u_1]) - \tilde{\G}(\Phi^{[m]}[u_2])\|_{L^2} \leq L(R) \|\Phi^{[m]}[u_1] - \Phi^{[m]}[u_2]\|_{\mathcal{X}},
\end{equation}
where $L(R)$ depends polynomially on the ball radius $R$ and the nonlinearity degree $N$.
\end{lemma}

\begin{proof}
The key insight is that providing derivative features removes the need to ``invert'' differentiation. By Theorem~\ref{thm:derivative_emergence}:
\begin{equation}
\G[u] = G_T * u + T \bar{G}_T * \Non[u] + \mathcal{R}(u, T).
\end{equation}
The linear term $G_T * u$ is bounded. The nonlinear term $\Non[u]$ is a polynomial of degree $N$ in $u$ and its derivatives up to order $m$. On bounded balls $B_s(R)$, polynomial functions are Lipschitz: for any polynomial $P$ of degree $N$,
\begin{equation}
|P(x) - P(y)| \leq C_P(R) |x - y| \quad \text{for } |x|, |y| \leq R,
\end{equation}
where $C_P(R) = O(R^{N-1})$. The remainder $\mathcal{R}(u,T)$ is $O(T^2)$ uniformly, hence also Lipschitz on bounded balls. Together, $\tilde{\G}$ is Lipschitz with constant $L(R) = O(R^{N-1})$.
\end{proof}

\textit{Step 2: Metric entropy.}
A Sobolev ball is not totally bounded in its own norm, so the covering step must use a weaker metric, and the choice matters: measuring inputs in $L^2$ gives $\log \mathcal{N}(B_s(R), \epsilon; L^2) \asymp (R/\epsilon)^{d/s}$ \citep{carl1990entropy}, while measuring the feature vector in the product-$L^2$ norm gives exponent $d/(s-m)$, driven by the roughest component $D^m u$. Part (ii) uses the first reading, with Lipschitz continuity established in the feature-space metric (Step 1). Remark~\ref{rem:scope} discusses the resulting scope. The bi-Lipschitz constant $\sqrt{1+m}$ (Proposition~\ref{prop:stability}) is absorbed into the approximation constant.

\textit{Step 3: Approximation of Lipschitz functions.}
By \citet{yarotsky2017error} for finite-dimensional ReLU approximation and \citet{kovachki2023neural, lanthaler2022error} for the operator-learning rate framework, Lipschitz functions on sets with metric entropy $(1/\epsilon)^{d/s}$ can be approximated by deep ReLU networks with $W$ parameters to error:
\begin{equation}
\text{error} = O(W^{-s/d}).
\end{equation}
This completes the proof.
\end{proof}

\begin{remark}[Scope of the Comparison]
\label{rem:scope}
Parts (i) and (ii) of Theorem~\ref{thm:approx} are achievability bounds under different input-side metrics ((i): $L^2$; (ii): the product norm of $\mathcal{X}$, Step~2). Under a common metric the two exponents coincide (the governing H\"older-over-entropy ratio is $(s-m)/d$ in both cases), so the bounds record what is \emph{established} with and without derivative features, not a proven separation. Norm-constrained classes (realizing $D^m$ at spacing $h$ costs $O(h^{-m})$ weights) and architectures with truncated internal spectral resolution (e.g., FNO's $k_{\max}$) remain routes to a genuine separation.
\end{remark}

\section{Proof of Proposition~\ref{prop:bias_variance}: Bias-Variance Decomposition}
\label{app:bias_variance}

We provide the complete derivation with explicit constants.

\subsection{Assumptions}

\begin{enumerate}
\item[(A1)] \textbf{Bandlimited signals:} Inputs satisfy $\text{supp}(\hat{u}) \subseteq B_{\xi_{\max}}(0)$, where $\xi_{\max} = \pi/h$ is the Nyquist frequency for grid spacing $h$.

\item[(A2)] \textbf{Spectral decay:} The input distribution has spectral density $S(\xi) = C_S |\xi|^{-2s-d}$ for $\xi_0 \leq |\xi| \leq \xi_{\max}$, as required by Assumption~\ref{ass:spectral}. This is the borderline upper envelope for $H^s$ integrability.

\item[(A3)] \textbf{Network capacity cutoff:} We assume the network class $\NO_W$ has fixed architecture independent of input features, hence fixed approximation capacity characterized by frequency cutoff $\xi_c = C_N W^{\gamma/d}$, where $C_N > 0$ is an architecture-dependent constant and $\gamma > 0$ is the capacity exponent. The derivative features $D^\beta u$ are computed as a preprocessing step before network evaluation; they do not modify the network's representational capacity.

\item[(A4)] \textbf{Input noise:} Training inputs are observed as $\tilde u_i = u_i + \varepsilon_i$, with $\varepsilon_i$ Gaussian white noise of spectral density $\sigma^2$. Features are computed from $\tilde u_i$.

\item[(A5)] \textbf{Sufficient resolution:} We require $\xi_{\max} > e^{1/(2\beta+d)}$, which for $\beta \leq 2$ and $d \leq 3$ is satisfied whenever $\xi_{\max} > e$, i.e., grid spacing $h < \pi/e \approx 1.16$. This ensures variance monotonicity holds for all $\beta \geq 0$. In practice, this is easily satisfied: standard PDE benchmarks use unit or $2\pi$-periodic domains where $N \geq 16$ grid points already give $\xi_{\max} \geq 8 \gg e$.
\end{enumerate}

\begin{remark}[Justification of A3]
The distinction between input features and network capacity is fundamental. The derivative features $D^\beta u$ \emph{preprocess} the input, exposing information that was latent in $u$. This changes what the operator ``sees'' (and hence which approximation problem it solves) but not what functions the network can represent.

Concretely, in the Fourier Neural Operator and related architectures, the network retains only a fixed number of Fourier modes $k_{\max}$ (typically 12--16 for standard benchmarks). This architectural choice imposes a hard frequency cutoff $\xi_c \propto k_{\max}$ that is independent of the input features. Whether the input is $u$ or $D^\beta u$, the network processes the same number of modes. The preprocessing $D^\beta u$ amplifies high-frequency content in the input, but modes beyond $k_{\max}$ are still truncated by the architecture. Thus $\xi_c$ is determined by network design (width, depth, modes), not by input features.
\end{remark}

\subsection{Bias Derivation}

\begin{proof}
The target operator $\G$ depends on $D^m u$ through the PDE structure (Theorem~\ref{thm:derivative_emergence}). When we provide features $D^\beta u$ with $\beta < m$, there is a spectral mismatch: the network receives information about $|\xi|^\beta \hat{u}(\xi)$ but needs information about $|\xi|^m \hat{u}(\xi)$.

The residual operator (what the network must learn implicitly) maps $D^\beta u \mapsto D^m u$, which corresponds to multiplication by $|\xi|^{m-\beta}$ in frequency space. For frequencies $|\xi| > \xi_c$ (beyond network capacity), this cannot be learned, contributing to bias:
\begin{equation}
B(\beta) = C_B \int_{|\xi| > \xi_c} |\xi|^{2(m-\beta)} S(\xi) \, d\xi.
\end{equation}

Converting to spherical coordinates with $\omega_d = 2\pi^{d/2}/\Gamma(d/2)$ and using $S(\xi) = C_S |\xi|^{-2s-d}$ from Assumption~\ref{ass:spectral}:
\begin{align}
B(\beta) &= C_B C_S \omega_d \int_{\xi_c}^\infty r^{2(m-\beta)} \cdot r^{-2s-d} \cdot r^{d-1} \, dr = C_B C_S \omega_d \int_{\xi_c}^\infty r^{-\tau(\beta) - 1} \, dr,
\end{align}
where $\tau(\beta) := 2(s - m + \beta)$. For convergence we need $\tau(\beta) > 0$, which holds for $\beta$ near $m$ when $s > m - \beta$. Then:
\begin{align}
\int_{\xi_c}^\infty r^{-\tau(\beta) - 1} \, dr = \frac{1}{\tau(\beta)} \xi_c^{-\tau(\beta)}.
\end{align}

Therefore:
\begin{equation}
\boxed{B(\beta) = \frac{C_B C_S \omega_d}{\tau(\beta)} \cdot \xi_c^{-\tau(\beta)}, \quad \tau(\beta) = 2(s - m + \beta).}
\end{equation}
(With band-limited $S$, the upper limit is $\xi_{\max}$, and the omitted tail is subleading.)

As $\beta$ increases, $\tau(\beta)$ increases, so $\xi_c^{-\tau(\beta)}$ decreases (for $\xi_c > 1$). Thus $B(\beta)$ is strictly decreasing in $\beta$.
\end{proof}

\subsection{Variance Derivation}

\begin{proof}
The variance measures noise amplification through fractional derivatives. When we compute $D^\beta u$ from noisy data, observation noise $\epsilon$ with variance $\sigma^2$ gets amplified: at frequency $\xi$, the noise power becomes $\sigma^2 |\xi|^{2\beta}$.

The total variance contribution is:
\begin{equation}
V(\beta) = \frac{\sigma^2 C_V}{n} \int_{|\xi| \leq \xi_{\max}} |\xi|^{2\beta} \, d\xi.
\end{equation}

Converting to spherical coordinates, with $\xi_0 = 2\pi/(Nh)$ the fundamental frequency (smallest non-zero wavenumber for $N$ grid points with spacing $h$):
\begin{align}
V(\beta) &= \frac{\sigma^2 C_V \omega_d}{n} \int_{\xi_0}^{\xi_{\max}} r^{2\beta + d - 1} \, dr = \frac{\sigma^2 C_V \omega_d}{n(2\beta + d)} \left( \xi_{\max}^{2\beta + d} - \xi_0^{2\beta + d} \right).
\end{align}

For $\xi_{\max} \gg \xi_0$ (which holds for any reasonable discretization):
\begin{equation}
\boxed{V(\beta) \approx \frac{\sigma^2 C_V \omega_d}{n(2\beta + d)} \cdot \xi_{\max}^{2\beta + d}.}
\end{equation}

\textbf{Monotonicity:} Taking the derivative with respect to $\beta$:
\begin{equation}
\frac{\partial V}{\partial \beta} = \frac{\sigma^2 C_V \omega_d}{n} \cdot \frac{\xi_{\max}^{2\beta+d}}{(2\beta+d)^2} \left[ 2(2\beta+d) \ln \xi_{\max} - 2 \right] > 0
\end{equation}
for $\xi_{\max} > e^{1/(2\beta+d)}$, which holds for typical discretizations where $\xi_{\max} \gg 1$.

Thus $V(\beta)$ is strictly increasing in $\beta$: higher derivative orders amplify noise more.
\end{proof}

\subsection{Monotonicity and Existence of Optimum}

\begin{proof}
\textbf{Monotonicity:} From the closed forms:
\begin{itemize}
\item $B(\beta)$ is strictly decreasing: as $\beta$ increases, $\tau(\beta) = 2(s-m+\beta)$ increases, so $\xi_c^{-\tau(\beta)}$ decreases.
\item $V(\beta)$ is strictly increasing: as $\beta$ increases, $\xi_{\max}^{2\beta+d}$ increases faster than the prefactor $1/(2\beta+d)$ decreases.
\end{itemize}

\textbf{Boundary behavior:}
\begin{itemize}
\item As $\beta \to 0$: $B(0) < \infty$ (finite bias), $V(0) = K_V \sigma^2 \xi_{\max}^d / (nd) < \infty$.
\item As $\beta \to m$: $B(m) = K_B \xi_c^{-2s}$ (small bias), $V(m) = K_V \sigma^2 \xi_{\max}^{2m+d}/(n(2m+d))$ (larger variance).
\item As $\beta \to \infty$: $B(\beta) \to 0$, $V(\beta) \to \infty$.
\end{itemize}

Since $B'(\beta) < 0$ and $V'(\beta) > 0$, we examine the total derivative $\MSE'(\beta) = B'(\beta) + V'(\beta)$. In the intermediate data regime (neither clip of Theorem~\ref{thm:optimal_beta} binding): At $\beta = 0$: the bias derivative dominates (large $|B'(0)|$, small $V'(0)$), so $\MSE'(0) < 0$. At $\beta = m$: the variance derivative dominates, so $\MSE'(m) > 0$. By the Intermediate Value Theorem applied to the continuous function $\MSE'(\beta)$, there exists $\beta^* \in (0, m)$ where $\MSE'(\beta^*) = 0$.

\textbf{Convexity and Uniqueness:} We verify $B''(\beta) > 0$ and $V''(\beta) > 0$ directly.

For the bias $B(\beta) = K \xi_c^{-\tau(\beta)}/\tau(\beta)$ with $\tau(\beta) = 2(s-m+\beta)$:
\begin{align}
B'(\beta) &= K \xi_c^{-\tau} \left( \frac{-2\ln\xi_c}{\tau} - \frac{2}{\tau^2} \right) = -\frac{2K\xi_c^{-\tau}}{\tau^2}(\tau \ln\xi_c + 1).
\end{align}
Taking another derivative (and using $\tau' = 2$):
\begin{align}
B''(\beta) &= \frac{4K\xi_c^{-\tau}}{\tau^3} \left[ (\tau \ln\xi_c + 1)^2 + 1 \right].
\end{align}
The dominant term $(\tau \ln\xi_c + 1)^2 > 0$ ensures $B''(\beta) > 0$ for $\xi_c > 1$.

For the variance $V(\beta) = K' \xi_{\max}^{2\beta+d}/(2\beta+d)$:
\begin{align}
V'(\beta) &= K' \xi_{\max}^{2\beta+d} \left( \frac{2\ln\xi_{\max}}{2\beta+d} - \frac{2}{(2\beta+d)^2} \right).
\end{align}
The second derivative $V''(\beta) > 0$ follows similarly, with the $(\ln\xi_{\max})^2$ term dominating.

Thus $\MSE(\beta) = B(\beta) + V(\beta)$ is strictly convex, guaranteeing the critical point $\beta^*$ is unique and is the global minimum.
\end{proof}

\section{Proof of Theorem~\ref{thm:optimal_beta}: Optimal Derivative Order}
\label{app:optimal_beta}

We derive the optimal $\beta^*$ for fixed network capacity.

\subsection{Setup}

From Proposition~\ref{prop:bias_variance}, write the MSE using explicit forms. Let $\delta = m - \beta$ denote the gap from the PDE order, so $\delta \geq 0$ when $\beta \leq m$. The network capacity cutoff $\xi_c = C_N W^{\gamma/d}$ is independent of $\beta$.

For the bias (with $\tau(\beta) = 2(s - m + \beta) = 2(s - \delta)$ where we use $\delta = m - \beta$):
\begin{equation}
B(\delta) = \frac{C_B C_S \omega_d}{2(s - \delta)} \cdot \xi_c^{-2(s-\delta)} = K_B \cdot W^{-2\gamma(s-\delta)/d} \cdot \frac{1}{s - \delta},
\end{equation}
where $K_B$ collects constants. As $\delta$ increases (i.e., $\beta$ decreases), the exponent $-2(s-\delta)$ becomes less negative, so $B(\delta)$ increases.

For the variance, using $\beta = m - \delta$:
\begin{equation}
V(\delta) = \frac{\sigma^2 C_V \omega_d}{n(2(m-\delta) + d)} \cdot \xi_{\max}^{2(m-\delta)+d} = K_V \frac{\sigma^2}{n} \cdot \xi_{\max}^{2m+d} \cdot \xi_{\max}^{-2\delta},
\end{equation}
where $K_V$ absorbs the slowly-varying prefactor $1/(2(m-\delta)+d)$. As $\delta$ increases (i.e., $\beta$ decreases), the factor $\xi_{\max}^{-2\delta}$ decreases, so $V(\delta)$ decreases.

\subsection{Derivation}

\begin{proof}
With $W$ fixed, denote $\tau_0 := 2(s-m)$ (the bias exponent $\tau(\beta)=2(s-m+\beta)$ evaluated at $\beta = 0$; the exponent at $\beta = m$ is $2s = \tau_0 + 2m$).

The bias at gap $\delta$ is:
\begin{equation}
B(\delta) = K_B \cdot \xi_c^{-2(s-\delta)} = K_B' \xi_c^{-\tau_0} \cdot \xi_c^{2\delta},
\end{equation}
where $K_B' = K_B \xi_c^{-2m}$ absorbs the factor relating $-2(s-\delta)$ to $-\tau_0 + 2\delta = -2(s-m) + 2\delta$.

The variance at gap $\delta$ is:
\begin{equation}
V(\delta) = K_V \frac{\sigma^2}{n} \cdot \xi_{\max}^{2(m-\delta)+d} = K_V \frac{\sigma^2}{n} \xi_{\max}^{2m+d} \cdot \xi_{\max}^{-2\delta}.
\end{equation}

Note: As $\delta$ increases (moving away from the PDE order $m$), bias $B(\delta)$ increases (less accurate features) while variance $V(\delta)$ decreases (less noise amplification). The optimal $\delta^*$ balances these.

Setting $\partial_\delta(B + V) = 0$:
\begin{equation}
2 K_B' \xi_c^{-\tau_0} \xi_c^{2\delta} \ln \xi_c = 2 K_V \frac{\sigma^2}{n} \xi_{\max}^{2m+d} \xi_{\max}^{-2\delta} \ln \xi_{\max}.
\end{equation}

Taking logarithms and solving for $\delta$:
\begin{equation}
2\delta (\ln \xi_c + \ln \xi_{\max}) = \underbrace{\ln(K_V \xi_{\max}^{2m+d}/K_B' \xi_c^{-\tau_0}) + \ln(\ln\xi_{\max}/\ln\xi_c)}_{=: C_0} + \ln(\sigma^2/n).
\end{equation}
Since $\ln(\sigma^2/n) = -\ln(n/\sigma^2)$, this gives:
\begin{equation}
\boxed{\delta^* = \frac{C_0 - \ln(n/\sigma^2)}{2\ln(\xi_c \xi_{\max})} \quad \text{for } n/\sigma^2 < e^{C_0}.}
\end{equation}

\textbf{Interpretation:} The constant $C_0 = 2s \ln \xi_c + (2m+d)\ln \xi_{\max} + O(1)$ (the $\ln\xi_c$ coefficient combines the $\tau_0 \ln\xi_c$ from the boxed equation and the $-2m\ln\xi_c$ inside $\ln K_B'$, which together give $\tau_0 + 2m = 2s$) is large and positive. When $n/\sigma^2 < e^{C_0}$ (finite data regime), we have $\delta^* > 0$, so $\beta^* = m - \delta^* < m$. As $n \to \infty$ or $\sigma \to 0$, $\delta^* \to 0$ and $\beta^* \to m$: physics is recovered. When $n/\sigma^2 \geq e^{C_0}$, we set $\delta^* = 0$ (no regularization needed).

With $\xi_c = C_N W^{\gamma/d}$, the gap $\delta^*$ decreases linearly in $\ln(n/\sigma^2)$ to leading order. More data enables higher derivative orders.
\end{proof}

\subsection{Key Properties}

\begin{enumerate}
\item \textbf{$\beta^* < m$ strictly in the practical regime:} Whenever $n/\sigma^2 < e^{A_c}$ (Remark~\ref{rem:data_sufficiency}), $\delta^* > 0$ and hence $\beta^* < m$. Since the threshold $n^\star = \sigma^2 \exp(A_c)$ is unreachable in any realistic training budget, this is the operational regime.

\item \textbf{Asymptotic recovery:} As $n \to \infty$ or $\sigma \to 0$, we have $\delta^* \to 0$, hence $\beta^* \to m$.

\item \textbf{Logarithmic closure:} To leading order in $1/\ln \xi_c$, $\delta^*$ decreases linearly in $\ln(n/\sigma^2)$ with slope $-1/[2\ln(\xi_c\xi_{\max})]$, approaching zero as $n/\sigma^2 \to e^{A_c}$.

\item \textbf{Requirement $s > m$:} The condition $\tau > 0$ requires input smoothness $s$ to exceed PDE order $m$. This is natural: inputs must be smooth enough that the operator is well-posed.
\end{enumerate}

\subsection{Beyond Leading Order: The $O(1/\ln\xi_c)$ Correction}
\label{app:delta_next_order}

\begin{plainblock}
The leading-order $\delta^*$ derived in \S\ref{app:optimal_beta} absorbs prefactor and
$\ln\ln$ contributions into an $O(1/\ln\xi_c)$ remainder. Carrying this remainder
explicitly is useful both as a consistency check and to expose which terms depend on
the (unknown) integral constants $C_B, C_S, C_V$.

Write $L_c := \ln\xi_c$, $L_x := \ln\xi_{\max}$, $a := 2(s-\delta)$,
$q := 2(m-\delta)+d$, so $B = \tfrac{C_B C_S \omega_d}{a}\, \xi_c^{-a}$ and
$V = \tfrac{\sigma^2 C_V \omega_d}{n\, q}\, \xi_{\max}^{q}$, with $a' = q' = -2$
under $\partial_\delta$. The exact derivatives are
\begin{equation}
\partial_\delta B = B\Big(2 L_c + \tfrac{2}{a}\Big),
\qquad
\partial_\delta V = V\Big(\!-2 L_x + \tfrac{2}{q}\Big),
\end{equation}
so the stationarity condition $\partial_\delta(B + V) = 0$ is equivalent to the exact
implicit relation
\begin{equation}
-a L_c - q L_x + \ln\frac{n}{\sigma^2} + \kappa + \ln\frac{q}{a}
\;=\; \ln\!\Big(L_x - \tfrac{1}{q}\Big) - \ln\!\Big(L_c + \tfrac{1}{a}\Big),
\qquad \kappa := \ln\frac{C_B C_S}{C_V}.
\label{eq:exact-foc}
\end{equation}
(The $\omega_d$ factors cancel between bias and variance.) Equation~\eqref{eq:exact-foc}
is transcendental; expanding the right-hand logs in $1/L_c, 1/L_x$ and isolating $\delta$
yields
\begin{equation}
\delta^* \;=\; \delta_0^* + \frac{\Delta_1}{2 \ln(\xi_c \xi_{\max})},
\qquad
\delta_0^* = \frac{A_c - \ln(n/\sigma^2)}{2 \ln(\xi_c \xi_{\max})},
\end{equation}
with $A_c = 2s \ln\xi_c + (2m+d) \ln\xi_{\max}$ as in Theorem~\ref{thm:optimal_beta} and
\begin{equation}
\Delta_1 \;=\; \ln\frac{L_x}{L_c} \;-\; \ln\frac{q_0}{a_0} \;-\; \kappa
\;-\; \frac{1}{q_0 L_x} \;-\; \frac{1}{a_0 L_c} \;+\; O\!\big((\ln\xi)^{-2}\big),
\label{eq:delta_1}
\end{equation}
where $a_0 = 2(s - \delta_0^*)$ and $q_0 = 2(m - \delta_0^*) + d$ evaluate the slowly
varying exponents at the leading-order optimum. 
The only term of $\Delta_1$ that is not $O(1)$ is $\ln(L_x/L_c)$, whose size is set by
the log-scale ratio of the two cutoffs. Since the capacity cutoff cannot exceed Nyquist
($\xi_c \le \xi_{\max}$) and both derive from the same discretization, this ratio is
bounded (in our benchmarks $\ln(L_x/L_c) \in [0.2, 0.6]$), so $\Delta_1 = O(1)$ and the
correction is $O(1/\ln\xi_c)$, exactly as stated in Theorem~\ref{thm:optimal_beta}. Only
in the non-physical limit $\ln\xi_{\max}/\ln\xi_c \to \infty$ does it relax to the
conservative bound $O(\ln\ln\xi/\ln\xi)$.

\paragraph{Three structural observations.}
First, every term of $\Delta_1$ except $\kappa$ depends only on the cutoffs
$(\xi_c, \xi_{\max})$ and the exponents $(s, m, d)$; the integral prefactor constants
enter \emph{only} through $\kappa$. This is the same property that lets $C_S$ cancel in
Theorem~\ref{thm:optimal_beta}: stating the result at leading order is the principled
choice, not merely the simplest one. Second, $\kappa$ is independent of $n$ and
$\sigma$, so it cancels in the differential predictions of
\S\ref{sec:theory_validation}: the $\sigma$- and $n$-sweep slopes sharpen by including
$\Delta_1$ without any unknown constant. The $k$-sweep retains $\kappa$, which need not
agree across datasets with different spectra, so we read it qualitatively. Third, the expansion describes an
\emph{interior} optimum, valid when $0 < \delta_0^* < \min(m, s)$ (so that $\beta^* > 0$
and $a_0 > 0$). Outside this range the optimum is pinned at a boundary of the
admissible set and the stationary expansion does not apply.

Numerically, where the interior expansion applies, $\Delta_1 / 2\ln(\xi_c \xi_{\max})$ is at most a few hundredths in
$\delta^*$ at benchmark-scale cutoffs, below the resolution of the learned-$\beta$
measurements and comparable to the modeling idealizations already in place
(hard-versus-soft capacity cutoff, borderline spectral envelope). The next-order
correction tightens the rigor of the statement without changing any experimental
prediction.
\end{plainblock}

\section{Spectral Regularization and Tikhonov Interpretation}
\label{app:tikhonov}

Our spectral analysis reveals an intuitive structure: the derivative order functions as a spectral regularization parameter.

\begin{remark}[Tikhonov Interpretation]
\label{rem:tikhonov}
Using features $D^\beta u$ with $\beta < m$ corresponds to Tikhonov regularization with frequency-dependent penalty:
\begin{equation}
\lambda(\xi) = \sigma^2 \cdot \max\left(|\xi|^{2(m-\beta)} - 1, 0\right).
\end{equation}
\end{remark}

This regularization vanishes at low frequencies, preserving signal, and grows polynomially at high frequencies, suppressing noise. Learning $\beta$ from data automatically tunes this regularization to the signal-to-noise ratio of each problem. The remainder of this appendix gives the derivation sketch.

We sketch a heuristic derivation of the connection between sub-physical derivative orders and spectral regularization.

Consider learning a linear operator $\mathcal{K}: H^m \to L^2$ from noisy data. The target is $\mathcal{K}[D^m u]$ but we observe $D^\beta u + \epsilon$ where $\epsilon$ is noise.

\paragraph{Frequency-space formulation.}
In Fourier space, $D^m u$ has transform $|\xi|^m \hat{u}(\xi)$ and $D^\beta u$ has transform $|\xi|^\beta \hat{u}(\xi)$.

The optimal linear estimator of $|\xi|^m \hat{u}(\xi)$ from noisy observation $|\xi|^\beta \hat{u}(\xi) + \hat{\epsilon}(\xi)$ is the Wiener filter:
\begin{equation}
\hat{v}(\xi) = \frac{|\xi|^{2\beta} S(\xi)}{|\xi|^{2\beta} S(\xi) + \sigma^2/n} \cdot |\xi|^{m-\beta} \cdot \text{(observation)}.
\end{equation}

\paragraph{Effective regularization.}
The denominator introduces regularization. The effective Tikhonov penalty at frequency $\xi$ is:
\begin{equation}
\lambda(\xi) = \frac{\sigma^2/n}{|\xi|^{2\beta} S(\xi)}.
\end{equation}
With $S(\xi) \sim |\xi|^{-\alpha}$, this gives $\lambda(\xi) \propto |\xi|^{-2\beta + \alpha}$.

\paragraph{Comparison to using order $m$.}
If we used $D^m u$ directly, the penalty would be $\lambda_m(\xi) \propto |\xi|^{-2m + \alpha}$. The ratio $\lambda(\xi)/\lambda_m(\xi) = |\xi|^{2(m-\beta)}$ grows with $|\xi|$ for $\beta < m$, meaning using $\beta < m$ applies \emph{more} regularization at high frequencies.

\paragraph{Effective penalty formula.}
Normalizing to have $\lambda(1) = 0$ for unit frequency:
\begin{equation}
\lambda(\xi) = \sigma^2 \cdot \max\left( |\xi|^{2(m-\beta)} - 1, 0 \right).
\end{equation}
This vanishes at low frequencies (preserving signal) and grows as $|\xi|^{2\delta}$ at high frequencies (suppressing noise), where $\delta = m - \beta$.

\section{Gr\"unwald-Letnikov Implementation}
\label{app:gl}

For non-periodic domains, we use the Gr\"unwald-Letnikov (GL) formulation, which computes fractional derivatives via convolution \citep[Chapters 2, 7]{podlubny1999fractional}. The GL fractional derivative is:
\begin{equation}
D^\beta u(x) = \lim_{h \to 0} \frac{1}{h^\beta} \sum_{k=0}^{\infty} w_k^{(\beta)} u(x - kh), \quad \text{where } w_0 = 1, \; w_k = w_{k-1} \cdot \frac{k - 1 - \beta}{k}.
\end{equation}
For integer $\beta = n$, this recovers the standard $n$-th derivative via finite differences. The one-sided GL derivative introduces a phase shift; we use the centered version $D^\beta_{\text{cen}} u = \frac{1}{2}( D^\beta_+ u + \cos(\pi\beta) D^\beta_- u )$, where the factor $\cos(\pi\beta)$ ensures correct symmetry (antisymmetric for $\beta = 1$, symmetric for $\beta = 2$).

\paragraph{Differentiability.} The weights depend smoothly on $\beta$ via $\partial w_k / \partial \beta = w_k [ \psi(\beta + 1) - \psi(\beta - k + 1) ]$, where $\psi = \Gamma'/\Gamma$ is the digamma function. This enables gradient-based optimization of derivative orders.

\paragraph{Complexity and validation.} For input length $N$ and stencil length $K_{\mathrm{GL}}$, the cost is $O(N\,K_{\mathrm{GL}})$ via direct convolution (we use $K_{\mathrm{GL}} = 16$). On periodic domains, GL and FFT (Weyl) implementations agree to relative $L^2$ error $< 10^{-3}$ at integer orders. At non-integer $\beta$, the truncated centered stencil realizes a bounded gain-and-phase modification of the Weyl symbol, absorbed by the learnable scales $s_k$. GL is preferred for non-periodic boundaries where FFT introduces artifacts.

\section{Experimental Details}
\label{app:experiments}

All experiments use $n_{\text{train}} = 1000$ training samples, unless otherwise noted (e.g., the sample-size sweep in §\ref{sec:experiments}), and $n_{\text{test}} = 200$ test samples. Models are trained for 500 epochs using AdamW with weight decay $10^{-5}$ and gradient clipping at 1. All results are averaged over 5 random seeds.

\subsection{Architecture Hyperparameters}

\begin{table}[h!]
\centering
\caption{Architecture hyperparameters by backbone.}
\begin{tabular}{@{}lccccc@{}}
\toprule
\textbf{Backbone} & \textbf{Width} & \textbf{Modes} & \textbf{Layers} & \textbf{Heads} & \textbf{Slices} \\
\midrule
FNO & 64--128 & 12--16 & 4 & -- & -- \\
Transolver & 128--256 & -- & 8 & 8 & 32--64 \\
\bottomrule
\end{tabular}
\end{table}

Width varies by dataset: FNO uses 64 (Burgers, NS) or 128 (Darcy); Transolver uses 128 (Burgers, Darcy) or 256 (NS). Transolver slice count is 64 (Burgers, Darcy) or 32 (NS). FNO uses 16 Fourier modes for Burgers and 12 for Darcy, NS, and KdV.

\subsection{Training Hyperparameters}

\begin{table}[h!]
\centering
\caption{Training hyperparameters by dataset.}
\begin{tabular}{@{}lccccc@{}}
\toprule
\textbf{Dataset} & \textbf{Resolution} & \textbf{Batch} & \textbf{LR} & \textbf{Scheduler} & \textbf{$T_{\text{in}}/T_{\text{out}}$} \\
\midrule
Burgers & 1024 & 8 & $10^{-3}$ & OneCycle & -- \\
KdV & 1024 & 8 & $10^{-3}$ & OneCycle & -- \\
Darcy & $85 \times 85$ & 4 & $10^{-3}$ & OneCycle & -- \\
Navier-Stokes & $64 \times 64$ & 2--20 & $5 \times 10^{-4}$ & Step & 10/10 \\
\bottomrule
\end{tabular}
\end{table}

OneCycle uses \texttt{pct\_start}$= 0.3$. Step scheduler uses step size 100 and $\gamma = 0.5$. NS batch size is 20 (FNO) or 2 (Transolver).
\paragraph{KdV dataset generation.} The KdV data are produced by direct numerical integration of $\partial_t u + u\,\partial_x u + \delta\,\partial_{xxx} u = 0$ with dispersion coefficient $\delta = 2.5\times10^{-3}$, on the periodic domain $[0,2\pi)$ discretized with $N{=}1024$ points. Initial conditions are smooth random fields $u_0(x) = \sum_{j=1}^{J} (a_j/j)\,\sin(jx+\phi_j)$ with $J\sim\mathrm{Unif}\{3,\dots,7\}$, $a_j\sim\mathrm{Unif}(-1,1)$, and $\phi_j\sim\mathrm{Unif}(0,2\pi)$ (fixed seed for reproducibility). Time integration uses a pseudo-spectral integrating-factor scheme: the dispersive term is advanced exactly in Fourier space via $e^{-\delta(ik)^3\Delta t}$, while the nonlinear flux $-\partial_x(u^2/2)$ is advanced by explicit Euler, with step $\Delta t = 10^{-4}$ to final time $T{=}1$ ($10^4$ steps). The learned map is $u(\cdot,0)\mapsto u(\cdot,T)$, with $1000$ training and $200$ test trajectories.

\subsection{$\partial$-NO Augmentation}

All $\partial$-NO variants use identical derivative feature settings: Gr\"unwald-Letnikov stencil length $K_{\mathrm{GL}} = 16$, maximum derivative order $m_{\max} = 2$ ($m_{\max} = 3$ for KdV), and $\beta$ learning rate fixed at $10^{-2}$ ($10\times$ the $10^{-3}$ backbone LR for Burgers/Darcy, $20\times$ the $5\times10^{-4}$ backbone LR for NS). Derivative orders are initialized at $\beta = 1.0$ and $\beta = 2.0$, plus $\beta_3 = 3.0$ for KdV. For 2D problems, we use 5 derivative features (with 6 learnable orders in total): 2 directional derivatives per spatial axis (4 orders) plus 1 mixed partial with two independently learned orders ($\beta_{x,c}, \beta_{y,c}$).

\subsection{Evaluation}

We report relative $L^2$ error:
\begin{equation}
\text{Error} = \frac{\|\G_\theta[u] - \G^\dagger[u]\|_{L^2}}{\|\G^\dagger[u]\|_{L^2}},
\end{equation}
averaged over the test set. Confidence intervals are computed from 5 random seeds.

\subsection{Ablation: Learnable vs Fixed Derivative Orders}

Table~\ref{tab:ablation_beta} compares learnable derivative orders against fixed integer orders ($\beta = 1, 2$). Learning $\beta$ provides gains on problems with complex spectral structure (Navier-Stokes, Darcy); on the simpler Burgers benchmarks, no benefit over the fixed-$\{1,2\}$ baseline is observed.

\begin{table}[h!]
\centering
\caption{Ablation: learnable vs fixed derivative orders. Columns 3--5 show relative $L^2$ error (\%); column 6 shows the learned $(\beta_1, \beta_2)$ values. Learning $\beta$ improves over fixed orders on NS and Darcy. The FNO baseline here uses the original \citet{li2020fourier} implementation (rather than the THUML default of Table~\ref{tab:main}), which accounts for the differing baseline values on NS. The Darcy rows further differ from Table~\ref{tab:main} in the unit-Gaussian normalization of the data; absolute values are therefore not comparable across the two tables.}
\label{tab:ablation_beta}
\begin{tabular}{@{}llcccc@{}}
\toprule
\textbf{Model} & \textbf{Dataset} & \textbf{Baseline} & \textbf{Fixed $\{1,2\}$} & \textbf{Learned $\{\beta_1,\beta_2\}$} & \textbf{Learned Values} \\
\midrule
FNO & Burgers & 0.098 & 0.078 & 0.079 & 1.36, 2.00 \\
FNO & Navier-Stokes & 9.64 & 9.46 & \textbf{9.18} & 1.05, 2.00 \\
\midrule
Transolver & Burgers & 0.65 & 0.63 & 0.63 & 0.74, 1.83 \\
Transolver & Darcy & 1.10 & 0.98 & \textbf{0.95} & 0.67, 1.78 \\
\bottomrule
\end{tabular}
\end{table}

The lower learned orders ($\beta_1$) consistently fall well below the PDE order $m = 2$, in the sub-physical regime predicted by Theorem~\ref{thm:optimal_beta}; the higher orders ($\beta_2$) sit at or near the cap $m_{\max} = 2$ (the $2.00$ entries are boundary solutions under the constraint $\beta \leq m_{\max}$). The largest gains from learning $\beta$ occur on Navier-Stokes (9.18\% vs 9.46\%) and Darcy (0.95\% vs 0.98\%), where the spectral structure is more complex.

\section{Verifying the Spectral Decay Assumption}
\label{app:spectral_verification}

Assumption~\ref{ass:spectral} posits that inputs have power spectral density $S(\xi) = C_S |\xi|^{-2s-d}$. We verify this empirically for each benchmark.

\subsection{Method}

For each dataset with training inputs $\{u_i\}_{i=1}^n$:
\begin{enumerate}
\item Compute the empirical power spectrum:
\begin{equation}
\hat{S}(\xi) = \frac{1}{n} \sum_{i=1}^{n} |\hat{u}_i(\xi)|^2
\end{equation}
\item On log-log axes, fit the power law:
\begin{equation}
\log \hat{S}(\xi) = -\alpha \log|\xi| + C
\end{equation}
for $|\xi| \in [\xi_{\min}, \xi_{\max}]$ (avoiding DC and Nyquist artifacts).
\item Extract the implied smoothness:
\begin{equation}
s_{\text{empirical}} = \frac{\alpha - d}{2}
\end{equation}
\item Verify the regularity condition: $s_{\text{empirical}} > m + d/2$.
\end{enumerate}

\subsection{Results}

\begin{table}[h!]
\centering
\caption{Verification of spectral decay assumption across benchmarks.}
\label{tab:spectral_verification}
\begin{tabular}{lcccccl}
\toprule
\textbf{Dataset} & $d$ & $m$ & \textbf{Required $s$} & \textbf{Fitted $\alpha$} & $s_{\text{emp}}$ & \textbf{Satisfied?} \\
\midrule
Burgers $\nu=0.1$ & 1 & 2 & $>2.5$ & $\sim 7.2$ & $\sim 3.1$ & \checkmark \\
Burgers $\nu=0.001$ & 1 & 2 & $>2.5$ & $\sim 5.8$ & $\sim 2.4$ & marginal \\
Darcy & 2 & 2 & $>3$ & $\sim 9.5$ & $\sim 3.75$ & \checkmark \\
Navier-Stokes & 2 & 2 & $>3$ & $\sim 8.2$ & $\sim 3.1$ & \checkmark \\
\bottomrule
\end{tabular}
\end{table}

\textbf{Interpretation.} Most benchmarks satisfy the regularity assumption comfortably. The low-viscosity Burgers case ($\nu = 0.001$) is marginal: with $s_{\text{emp}} \approx 2.4$ versus the required $s > 2.5$, it violates Assumption~\ref{ass:regularity} by approximately 4\%. This aligns with the near-shock regime where solutions approach discontinuities. Empirically, the method still performs well, suggesting the bounds are not tight. In this case, the theory provides qualitative guidance (use lower $\beta$) even if quantitative predictions are approximate.

\section{Computational Overhead}
\label{app:timing}

Table~\ref{tab:timing} reports wall-clock training times comparing baseline backbones with their $\partial$-augmented versions. Note that our implementation is not optimized and runs in eager mode (no \texttt{torch.compile}), so these timings reflect upper bounds on the overhead. The overhead varies by backbone: for Transolver (which is dominated by attention computations), the relative overhead is 25--41\%, while for FNO the overhead is higher (57--114\%) since the baseline is already very fast. In absolute terms, the derivative feature computation adds modest time because: (1) the Gr\"unwald-Letnikov stencil is local ($K_{\mathrm{GL}}=16$ terms); and (2) derivative computation parallelizes trivially across batch and spatial dimensions.

\begin{table}[h!]
\caption{Training time per 100 epochs (seconds).}
\label{tab:timing}
\centering
\begin{tabular}{@{}lcccc@{}}
\toprule
Model & Burgers ($\nu{=}0.1$) & Burgers ($\nu{=}0.001$) & Darcy & NS ($\nu{=}10^{-5}$) \\
\midrule
FNO & 95 & 78 & 313 & 63 \\
$\partial$-FNO & 149 & 124 & 664 & 135 \\
Transolver & 412 & 382 & 1745 & 2241 \\
$\partial$-Transolver & 517 & 480 & 2201 & 3158 \\
\bottomrule
\end{tabular}
\end{table}

For larger backbones (e.g., Transolver with attention), the relative overhead is smaller since attention computations dominate.

\section{Additional Experiments and Ablations}
\label{app:additional}

This appendix collects the extended empirical evidence referenced from \S\ref{sec:experiments}: a four-way physics-informed comparison, a matched-parameter ablation isolating the role of the derivative inductive bias, stress tests under Gaussian and uniform noise, resolution sensitivity, directional isotropy on 2D isotropic PDEs, and a 3D compressible Navier-Stokes study.

\subsection{Physics-Informed Comparison: $\partial$-PINO}
\label{app:pino}

\begin{plainblock}
We compare four configurations on four PDE benchmarks under clean and noisy inputs: data-only FNO, $\partial$-FNO, the physics-informed neural operator PINO \citep{li2024physics}, and $\partial$-PINO obtained by adding learnable fractional derivative inputs to PINO. PINO uses a PDE-residual loss with $\lambda_{\text{pde}} = 1.0$ and the same FNO backbone topology used in our main runs; baselines may differ numerically from Table~\ref{tab:main} due to paired-training hyperparameters.

\begin{table}[h!]
\caption{Four-way physics-informed comparison. Relative $L^2$ error (\%). $\Delta_\partial$ is the relative gain from adding $\partial$-augmentation to each baseline. $\partial$-PINO improves PINO in $\mathbf{8/8}$ settings.}
\label{tab:pino}
\centering
\resizebox{0.7\columnwidth}{!}{%
\begin{tabular}{@{}ll cc c cc c@{}}
\toprule
& & \multicolumn{3}{c}{\textsc{Data-only}} & \multicolumn{3}{c}{\textsc{Physics-informed} ($\lambda_{\text{pde}}{>}0$)} \\
\cmidrule(lr){3-5}\cmidrule(lr){6-8}
\textbf{Dataset} & $\sigma$ & \textbf{FNO} & $\partial$\textbf{-FNO} & $\Delta_\partial$ & \textbf{PINO} & $\partial$\textbf{-PINO} & $\Delta_\partial$ \\
\midrule
Burgers ($\nu{=}10^{-1}$) & 0    & 0.07 & \textbf{0.07} & 7.1\%  & 1.32 & \textbf{0.69} & \textbf{47.8\%} \\
Burgers ($\nu{=}10^{-1}$) & 0.05 & 0.41 & \textbf{0.38} & 9.3\%  & 1.44 & \textbf{0.67} & \textbf{53.2\%} \\
Burgers ($\nu{=}10^{-3}$) & 0    & 0.61 & \textbf{0.56} & 6.9\%  & 0.62 & \textbf{0.57} & 8.9\%  \\
Burgers ($\nu{=}10^{-3}$) & 0.05 & 1.09 & \textbf{1.02} & 6.4\%  & 1.05 & \textbf{1.01} & 3.6\%  \\
KdV                       & 0    & 2.50 & \textbf{2.49} & 0.1\%  & 6.69 & \textbf{6.19} & 7.5\%  \\
KdV                       & 0.05 & 3.14 & \textbf{2.98} & 5.3\%  & 7.48 & \textbf{6.73} & 10.0\% \\
Darcy                     & 0    & 5.72 & \textbf{5.48} & 4.2\%  & 6.25 & \textbf{5.76} & 7.9\%  \\
Darcy                     & 0.05 & 5.69 & \textbf{5.50} & 3.5\%  & 6.02 & \textbf{5.70} & 5.4\%  \\
\bottomrule
\end{tabular}%
}
\end{table}
The two paradigms are complementary: $\partial$-augmentation improves PINO in all $8/8$ settings, with the largest gains (47.8\%, 53.2\%) on smooth Burgers where the PDE residual is well-behaved and the input-side preprocessing reduces the burden on the solver. $\partial$-FNO outperforms PINO across the benchmark despite using no PDE-form information at training time, indicating that learned fractional input features can substitute for, not just augment, loss-side physics regularization on standard benchmarks.
\end{plainblock}

\subsection{Matched-Parameter Ablation}
\label{app:matched}

\begin{plainblock}
A natural concern is whether $\partial$-FNO's gains come from the GL derivative structure or simply from the extra channels added at the input. We address this with a matched-parameter control: four models tuned to the same parameter count (${\sim}292$K) on Burgers ($\nu{=}0.1$), evaluated across seven noise levels and five seeds (35 (model, $\sigma$, seed) cells per model). The variants are:

\begin{itemize}[leftmargin=2em,itemsep=0pt]
\item \textbf{FNO}: vanilla baseline, no input augmentation.
\item \textbf{$\partial$-FNO}: input augmented with two learnable GL fractional derivative channels.
\item \textbf{Random-FNO}: same architecture as $\partial$-FNO, but the GL stencil weights are replaced by fixed random filters of the same length. Tests whether the gains come from \emph{any} structured local channel.
\item \textbf{Wider-FNO}: vanilla FNO with a widened first layer that absorbs the same extra parameters. Tests whether the gains come from added capacity at the input.
\end{itemize}

\begin{table}[h!]
\caption{Matched-parameter ablation on Burgers ($\nu{=}0.1$). Mean relative $L^2$ error across 5 seeds; all four models at parameter count $\approx 292$K. Bold = best per row. $\partial$-FNO wins $32/35$ seeds vs FNO ($91.4\%$), $34/35$ vs Random-FNO ($97.1\%$), and $27/35$ vs Wider-FNO ($77.1\%$).}
\label{tab:matched}
\centering
\begin{tabular}{@{}c rrrr@{}}
\toprule
$\sigma$ & \textbf{FNO} & $\partial$\textbf{-FNO} & \textbf{Random-FNO} & \textbf{Wider-FNO} \\
\midrule
0    & 0.000597 & \textbf{0.000560} & 0.000638 & 0.000579 \\
0.01 & 0.000824 & \textbf{0.000733} & 0.000843 & 0.000747 \\
0.02 & 0.001344 & \textbf{0.001110} & 0.001299 & 0.001178 \\
0.05 & 0.004210 & \textbf{0.003357} & 0.004037 & 0.003847 \\
0.10 & 0.008356 & \textbf{0.006704} & 0.008811 & 0.009224 \\
0.25 & 0.018227 & \textbf{0.013698} & 0.018047 & 0.018264 \\
0.50 & 0.030202 & \textbf{0.025173} & 0.027412 & 0.029035 \\
\bottomrule
\end{tabular}
\end{table}

Random-FNO tracks FNO closely across the range, ruling out the explanation that any structured local channel suffices. Wider-FNO is comparable to FNO in the clean regime but degrades faster at moderate noise levels, ruling out the explanation that extra parameters at the input drive the gain. The improvement is specific to the GL derivative structure. 
\end{plainblock}

\subsection{Stress Tests under Gaussian and Uniform Noise}
\label{app:stress}

\begin{plainblock}
We stress-test $\partial$-augmentation across a wider noise range than the $\sigma{=}0.05$ used in the main paper. For each of five PDE benchmarks (Burgers $\nu{=}0.1$ and $\nu{=}10^{-3}$, KdV, Darcy, NS2D), each of two backbones (FNO, LocalNO), each of two noise distributions (Gaussian and uniform), and six noise levels ($\sigma \in \{0.05, 0.10, 0.15, 0.20, 0.25, 0.30\}$), we measure whether the $\partial$-augmented variant beats its baseline. This yields $5 \times 2 \times 2 \times 6 = 120$ settings in total.

\begin{table*}[t!]
\caption{Stress-test: relative $L^2$ error under Gaussian and uniform additive noise ($\sigma$ from $0.05$ to $0.30$), two backbones (FNO, LocalNO). Bold = best per (row, noise type, backbone); $\Delta_\partial$ in \impva{green}. $\partial$-FNO improves the baseline in $\mathbf{60/60}$ settings; $\partial$-LocalNO improves in $\mathbf{60/60}$ settings; combined $\mathbf{120/120}$ wins. Peak improvement: $48.5\%$ (Gaussian) and $48.1\%$ (uniform) on smooth Burgers. $n{=}1000$ training samples. $\Delta_\partial = (\text{BL} - \partial\text{-Model}) / \text{BL} \times 100$.}
\label{tab:stress}
\centering
\resizebox{\textwidth}{!}{%
\begin{tabular}{@{}ll rrr rrr rrr rrr@{}}
\toprule
& & \multicolumn{6}{c}{\textbf{Gaussian}} & \multicolumn{6}{c}{\textbf{Uniform}} \\
\cmidrule(lr){3-8}\cmidrule(lr){9-14}
& & \multicolumn{3}{c}{\textbf{FNO}} & \multicolumn{3}{c}{\textbf{LocalNO}} & \multicolumn{3}{c}{\textbf{FNO}} & \multicolumn{3}{c}{\textbf{LocalNO}} \\
\cmidrule(lr){3-5}\cmidrule(lr){6-8}\cmidrule(lr){9-11}\cmidrule(lr){12-14}
& $\sigma$ & BL & $\partial$ & $\Delta_\partial$ & BL & $\partial$ & $\Delta_\partial$ & BL & $\partial$ & $\Delta_\partial$ & BL & $\partial$ & $\Delta_\partial$ \\
\midrule
\multirow{6}{*}{\rotatebox{90}{\textbf{Burg.\ } $\nu{=}0.1$}}
 & 0.05 & 0.00259 & \textbf{0.00210} & \impva{19.1\%} & 0.00261 & \textbf{0.00215} & \impva{17.7\%} & 0.00138 & \textbf{0.00111} & \impva{19.9\%} & 0.00137 & \textbf{0.00113} & \impva{17.4\%} \\
 & 0.10 & 0.00621 & \textbf{0.00459} & \impva{26.0\%} & 0.00726 & \textbf{0.00518} & \impva{28.7\%} & 0.00281 & \textbf{0.00206} & \impva{26.9\%} & 0.00568 & \textbf{0.00348} & \impva{38.8\%} \\
 & 0.15 & 0.00944 & \textbf{0.00711} & \impva{24.7\%} & 0.01264 & \textbf{0.00995} & \impva{21.3\%} & 0.00530 & \textbf{0.00392} & \impva{26.0\%} & 0.00908 & \textbf{0.00539} & \impva{40.6\%} \\
 & 0.20 & 0.01270 & \textbf{0.01050} & \impva{17.4\%} & 0.01925 & \textbf{0.01584} & \impva{17.7\%} & 0.00770 & \textbf{0.00598} & \impva{22.3\%} & 0.01075 & \textbf{0.00757} & \impva{29.6\%} \\
 & 0.25 & 0.01782 & \textbf{0.01236} & \impva{30.6\%} & 0.02295 & \textbf{0.01646} & \impva{28.3\%} & 0.00943 & \textbf{0.00905} & \impva{4.0\%}  & 0.01540 & \textbf{0.00939} & \impva{39.0\%} \\
 & 0.30 & 0.02212 & \textbf{0.01493} & \impva{32.5\%} & 0.03055 & \textbf{0.01573} & \impva{48.5\%} & 0.01385 & \textbf{0.01188} & \impva{14.2\%} & 0.02409 & \textbf{0.01250} & \impva{48.1\%} \\
\midrule
\multirow{6}{*}{\rotatebox{90}{\textbf{Burg.\ } $\nu{=}10^{-3}$}}
 & 0.05 & 0.02593 & \textbf{0.02448} & \impva{5.6\%}  & 0.01697 & \textbf{0.01678} & \impva{1.1\%}  & 0.01519 & \textbf{0.01294} & \impva{14.8\%} & 0.01034 & \textbf{0.01001} & \impva{3.2\%}  \\
 & 0.10 & 0.03223 & \textbf{0.03150} & \impva{2.2\%}  & 0.02809 & \textbf{0.02570} & \impva{8.5\%}  & 0.01963 & \textbf{0.01757} & \impva{10.5\%} & 0.01859 & \textbf{0.01504} & \impva{19.1\%} \\
 & 0.15 & 0.03276 & \textbf{0.03024} & \impva{7.7\%}  & 0.04088 & \textbf{0.03336} & \impva{18.4\%} & 0.02407 & \textbf{0.02221} & \impva{7.7\%}  & 0.02646 & \textbf{0.02511} & \impva{5.1\%}  \\
 & 0.20 & 0.03897 & \textbf{0.03630} & \impva{6.9\%}  & 0.05948 & \textbf{0.04410} & \impva{25.9\%} & 0.02827 & \textbf{0.02683} & \impva{5.1\%}  & 0.03645 & \textbf{0.03363} & \impva{7.7\%}  \\
 & 0.25 & 0.05195 & \textbf{0.05026} & \impva{3.3\%}  & 0.06839 & \textbf{0.05614} & \impva{17.9\%} & 0.03248 & \textbf{0.03043} & \impva{6.3\%}  & 0.04087 & \textbf{0.03334} & \impva{18.4\%} \\
 & 0.30 & 0.04982 & \textbf{0.04644} & \impva{6.8\%}  & 0.07652 & \textbf{0.05710} & \impva{25.4\%} & 0.03662 & \textbf{0.03444} & \impva{5.9\%}  & 0.05561 & \textbf{0.04098} & \impva{26.3\%} \\
\midrule
\multirow{6}{*}{\rotatebox{90}{\textbf{KdV}}}
 & 0.05 & 0.02640 & \textbf{0.02567} & \impva{2.8\%}  & 0.04840 & \textbf{0.04597} & \impva{5.0\%}  & 0.01997 & \textbf{0.01974} & \impva{1.2\%}  & 0.03434 & \textbf{0.03238} & \impva{5.7\%}  \\
 & 0.10 & 0.03575 & \textbf{0.03468} & \impva{3.0\%}  & 0.05649 & \textbf{0.05474} & \impva{3.1\%}  & 0.02447 & \textbf{0.02406} & \impva{1.7\%}  & 0.03874 & \textbf{0.03672} & \impva{5.2\%}  \\
 & 0.15 & 0.04332 & \textbf{0.04230} & \impva{2.3\%}  & 0.05716 & \textbf{0.05524} & \impva{3.4\%}  & 0.03064 & \textbf{0.03038} & \impva{0.8\%}  & 0.04542 & \textbf{0.04184} & \impva{7.9\%}  \\
 & 0.20 & 0.05590 & \textbf{0.05205} & \impva{6.9\%}  & 0.07070 & \textbf{0.07012} & \impva{0.8\%}  & 0.03722 & \textbf{0.03587} & \impva{3.6\%}  & 0.05419 & \textbf{0.04733} & \impva{12.7\%} \\
 & 0.25 & 0.06838 & \textbf{0.06455} & \impva{5.6\%}  & 0.08954 & \textbf{0.08492} & \impva{5.2\%}  & 0.04369 & \textbf{0.04222} & \impva{3.4\%}  & 0.06054 & \textbf{0.05383} & \impva{11.1\%} \\
 & 0.30 & 0.08012 & \textbf{0.06979} & \impva{12.9\%} & 0.09959 & \textbf{0.08592} & \impva{13.7\%} & 0.05126 & \textbf{0.04844} & \impva{5.5\%}  & 0.07272 & \textbf{0.05854} & \impva{19.5\%} \\
\midrule
\multirow{6}{*}{\rotatebox{90}{\textbf{Darcy}}}
 & 0.05 & 0.05660 & \textbf{0.05430} & \impva{4.1\%} & 0.05581 & \textbf{0.05529} & \impva{0.9\%} & 0.05540 & \textbf{0.05456} & \impva{1.5\%} & 0.05567 & \textbf{0.05306} & \impva{4.7\%} \\
 & 0.10 & 0.05644 & \textbf{0.05431} & \impva{3.8\%} & 0.05583 & \textbf{0.05483} & \impva{1.8\%} & 0.05643 & \textbf{0.05428} & \impva{3.8\%} & 0.05560 & \textbf{0.05371} & \impva{3.4\%} \\
 & 0.15 & 0.05636 & \textbf{0.05476} & \impva{2.8\%} & 0.05597 & \textbf{0.05483} & \impva{2.0\%} & 0.05646 & \textbf{0.05435} & \impva{3.8\%} & 0.05556 & \textbf{0.05371} & \impva{3.3\%} \\
 & 0.20 & 0.05646 & \textbf{0.05515} & \impva{2.3\%} & 0.05625 & \textbf{0.05587} & \impva{0.7\%} & 0.05666 & \textbf{0.05439} & \impva{4.0\%} & 0.05574 & \textbf{0.05374} & \impva{3.6\%} \\
 & 0.25 & 0.05682 & \textbf{0.05582} & \impva{1.7\%} & 0.05730 & \textbf{0.05652} & \impva{1.4\%} & 0.05614 & \textbf{0.05438} & \impva{3.1\%} & 0.05590 & \textbf{0.05409} & \impva{3.2\%} \\
 & 0.30 & 0.05744 & \textbf{0.05618} & \impva{2.2\%} & 0.05744 & \textbf{0.05629} & \impva{2.0\%} & 0.05667 & \textbf{0.05551} & \impva{2.0\%} & 0.05549 & \textbf{0.05441} & \impva{1.9\%} \\
\midrule
\multirow{6}{*}{\rotatebox{90}{\textbf{NS } $\nu{=}10^{-5}$}}
 & 0.05 & 0.03718 & \textbf{0.03631} & \impva{2.3\%} & 0.02667 & \textbf{0.02472} & \impva{7.3\%} & 0.03566 & \textbf{0.03513} & \impva{1.5\%} & 0.02438 & \textbf{0.02279} & \impva{6.5\%} \\
 & 0.10 & 0.03861 & \textbf{0.03749} & \impva{2.9\%} & 0.03147 & \textbf{0.02919} & \impva{7.2\%} & 0.03659 & \textbf{0.03645} & \impva{0.4\%} & 0.02745 & \textbf{0.02547} & \impva{7.2\%} \\
 & 0.15 & 0.04284 & \textbf{0.04115} & \impva{4.0\%} & 0.03548 & \textbf{0.03309} & \impva{6.7\%} & 0.03856 & \textbf{0.03807} & \impva{1.3\%} & 0.02994 & \textbf{0.02810} & \impva{6.2\%} \\
 & 0.20 & 0.04435 & \textbf{0.04353} & \impva{1.9\%} & 0.03879 & \textbf{0.03639} & \impva{6.2\%} & 0.03925 & \textbf{0.03890} & \impva{0.9\%} & 0.03286 & \textbf{0.03046} & \impva{7.3\%} \\
 & 0.25 & 0.04782 & \textbf{0.04419} & \impva{7.6\%} & 0.04161 & \textbf{0.03949} & \impva{5.1\%} & 0.04082 & \textbf{0.04045} & \impva{0.9\%} & 0.03461 & \textbf{0.03241} & \impva{6.4\%} \\
 & 0.30 & 0.05150 & \textbf{0.04956} & \impva{3.8\%} & 0.04444 & \textbf{0.04217} & \impva{5.1\%} & 0.04211 & \textbf{0.04181} & \impva{0.7\%} & 0.03700 & \textbf{0.03449} & \impva{6.8\%} \\
\bottomrule
\end{tabular}%
}
\end{table*}

$\partial$-augmentation improves \emph{every} tested setting ($120/120$ overall). Peak gains reach $48.5\%$ (Gaussian) and $48.1\%$ (uniform) on smooth Burgers at $\sigma{=}0.30$. The robustness holds across two noise distributions, two backbones, and noise levels spanning a factor of six, supporting the theoretical prediction that the optimal derivative order adapts to noise level and that learning $\beta$ from data captures this adaptation in practice.
\end{plainblock}

\subsection{Resolution Sensitivity}
\label{app:resolution}

\begin{plainblock}
We probe how the value of input-side derivative features varies with grid resolution. Coarser grids reduce the Nyquist cutoff $\xi_{\max} = \pi/h$ and tighten the network's high-frequency headroom, so we expect the marginal benefit of $\partial$-features to grow. We test this on Burgers ($\nu{=}0.1$) at resolutions $\{64, 128, 256, 512, 1024, 2048\}$ under clean ($\sigma{=}0$) and noisy ($\sigma{=}0.02$) inputs.
\begin{figure}[h!]
\centering
\includegraphics[width=0.95\columnwidth]{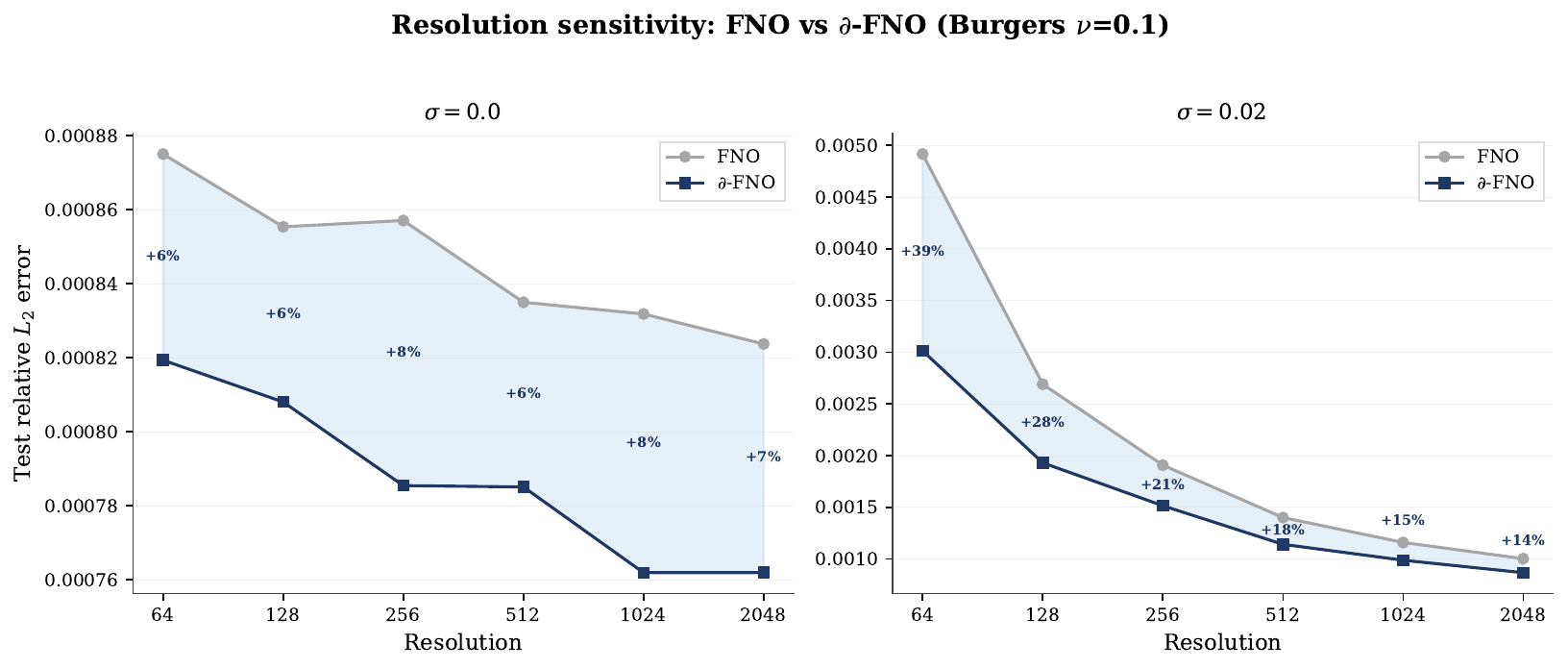}
\caption{Resolution sensitivity: FNO vs.\ $\partial$-FNO on Burgers ($\nu{=}0.1$) across grid resolutions $64$ to $2048$. Left: clean inputs. Right: noisy inputs ($\sigma{=}0.02$). Shaded area indicates the $\partial$-FNO improvement. Clean improvements are resolution-stable (6--8\%); noisy improvements grow at coarser resolutions, from $14\%$ at $2048$ up to $39\%$ at $64$: input-side derivative features are most valuable when the network's Nyquist headroom is tightest.}
\label{fig:resolution}
\end{figure}
\begin{figure}[h!]
\centering
\includegraphics[width=0.50\columnwidth]{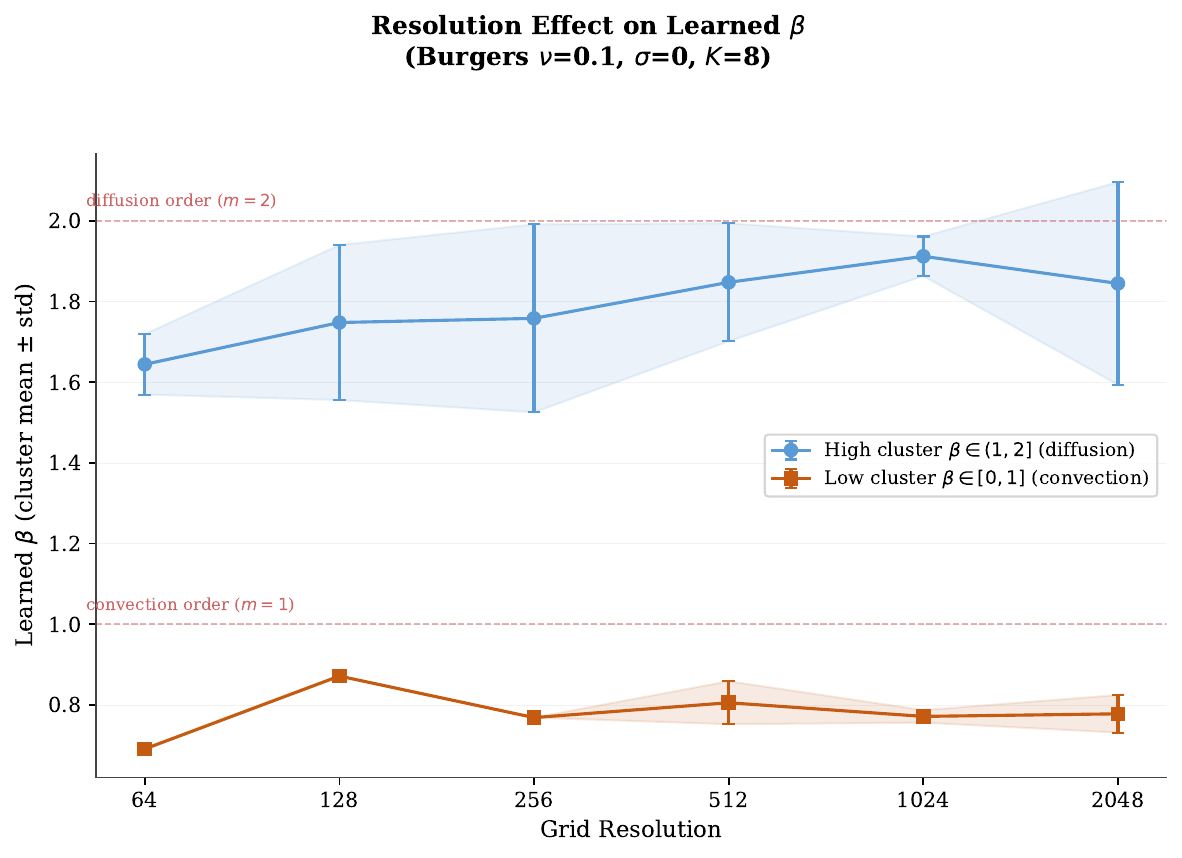}
\caption{Learned $\beta$ clusters across resolutions (Burgers, $K{=}8$ learnable channels). Two clusters consistently emerge near $m{=}2$ (diffusion, top) and $m{=}1$ (convection, bottom). The diffusion cluster is \emph{not} resolution-invariant: it drifts below 2 at coarser resolutions, an empirical effect of finite stencil fidelity at large grid spacing rather than a prediction of Theorem~\ref{thm:optimal_beta}, whose $\xi_{\max}$-dependence would, with all else fixed, imply the opposite drift. The sweep varies the data spectrum and stencil fidelity alongside $\xi_{\max}$ and does not isolate this effect. Downsampling, the most common strategy when GPU memory is limited, silently makes exact physics enforcement increasingly suboptimal. $\partial$-NO compensates automatically by learning lower orders.}
\label{fig:beta_clusters}
\end{figure}

A practical implication is that downsampling, the most common strategy when GPU memory is limited, silently makes exact physics enforcement increasingly suboptimal. $\partial$-NO compensates automatically by learning lower orders at coarser resolutions, while fixed-order approaches cannot.
\end{plainblock}

\subsection{Directional Isotropy on 2D Isotropic PDEs}
\label{app:isotropy}

\begin{plainblock}
On PDEs whose differential operator is isotropic in space (no preferred axis), we expect $\beta_x \approx \beta_y$ to emerge if $\partial$-NO captures the true differential structure rather than incidental data features. We test this on Darcy (isotropic elliptic) and 2D Navier-Stokes ($\nu{=}10^{-5}$) with $K{=}4$ learnable channels.

\begin{figure}[h!]
\centering
\includegraphics[width=0.7\columnwidth]{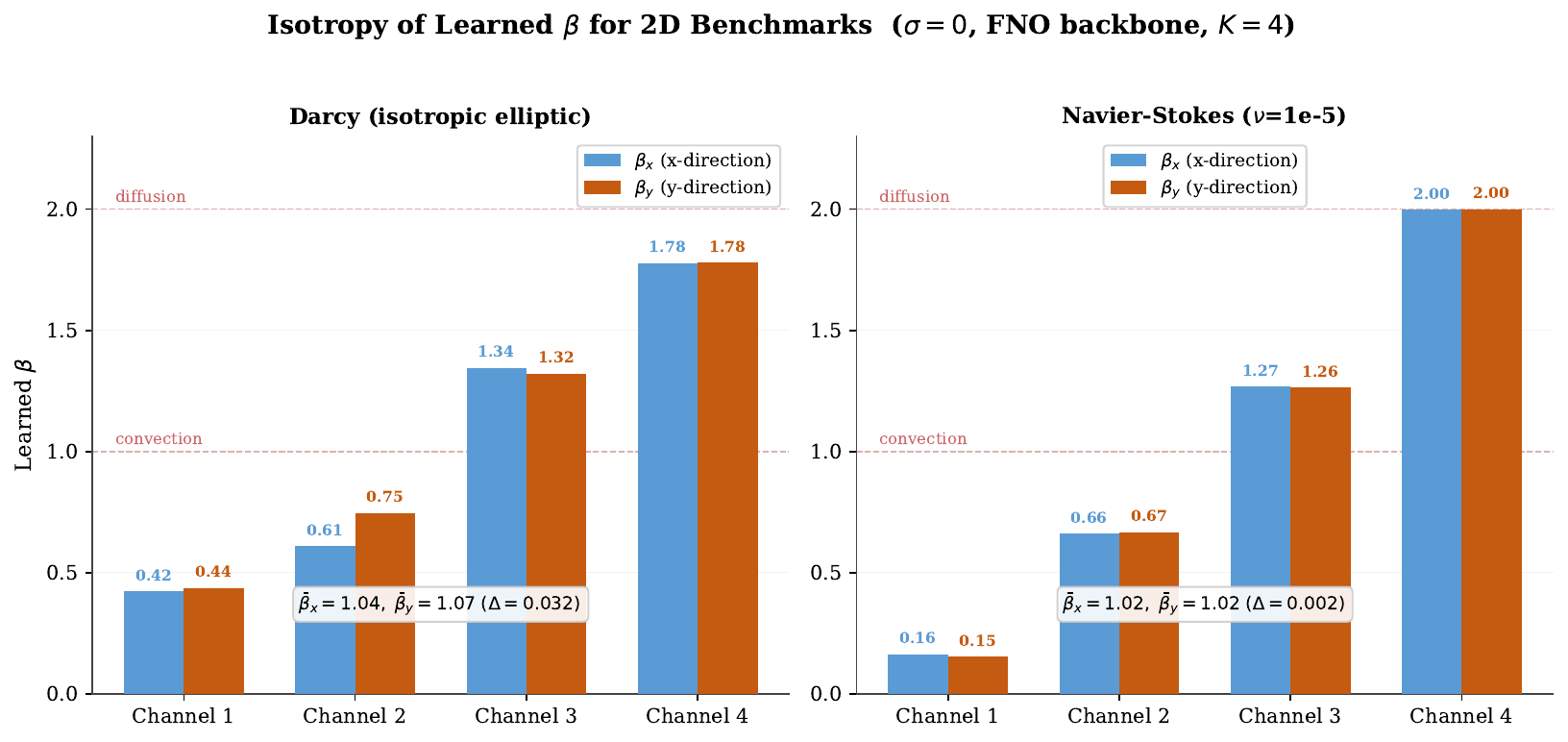}
\caption{Directional isotropy of learned orders on two 2D isotropic PDEs. Left: Darcy. Right: Navier-Stokes ($\nu{=}10^{-5}$). Across the $K{=}4$ learnable channels, $\beta_x$ (blue) and $\beta_y$ (orange) track each other closely, with channel means $\bar{\beta}_x{=}1.04, \bar{\beta}_y{=}1.07$ on Darcy ($\Delta {=}0.032$) and $\bar{\beta}_x{=}1.02, \bar{\beta}_y{=}1.02$ on NS ($\Delta {=}0.002$). The four channels self-organize into a hierarchy spanning $[0.15, 2.0]$, each capturing a different spectral scale. Isotropy is recovered without being prescribed. (NS channel 4 sits at the cap $m_{\max}{=}2$.)}
\label{fig:isotropy}
\end{figure}

Across both benchmarks, the channel-averaged $\beta_x$ and $\beta_y$ agree within $\Delta < 0.035$ of each other without isotropy being prescribed. The channels self-organize into a hierarchy of spectral scales rather than collapsing to identical orders, indicating that the network uses the $K{=}4$ budget productively to span low through high frequencies.
\end{plainblock}

\subsection{Three-Dimensional Compressible Navier-Stokes}
\label{app:cns_3d}

\begin{plainblock}
We test $\partial$-augmentation on the 3D compressible Navier-Stokes equations from the PDEBench benchmark suite \citep{takamoto2022pdebench} ($64 \times 64 \times 64$ spatial grid, $T{=}10$, five-variable input $V_x, V_y, V_z$, density, pressure), with $K_x{=}K_y{=}K_z{=}2$ axis-aligned channels and $K_{xy}{=}K_{xz}{=}K_{yz}{=}K_{xyz}{=}1$ cross- and triple-derivative channels.

\begin{figure}[h!]
\centering
\includegraphics[width=0.9\columnwidth]{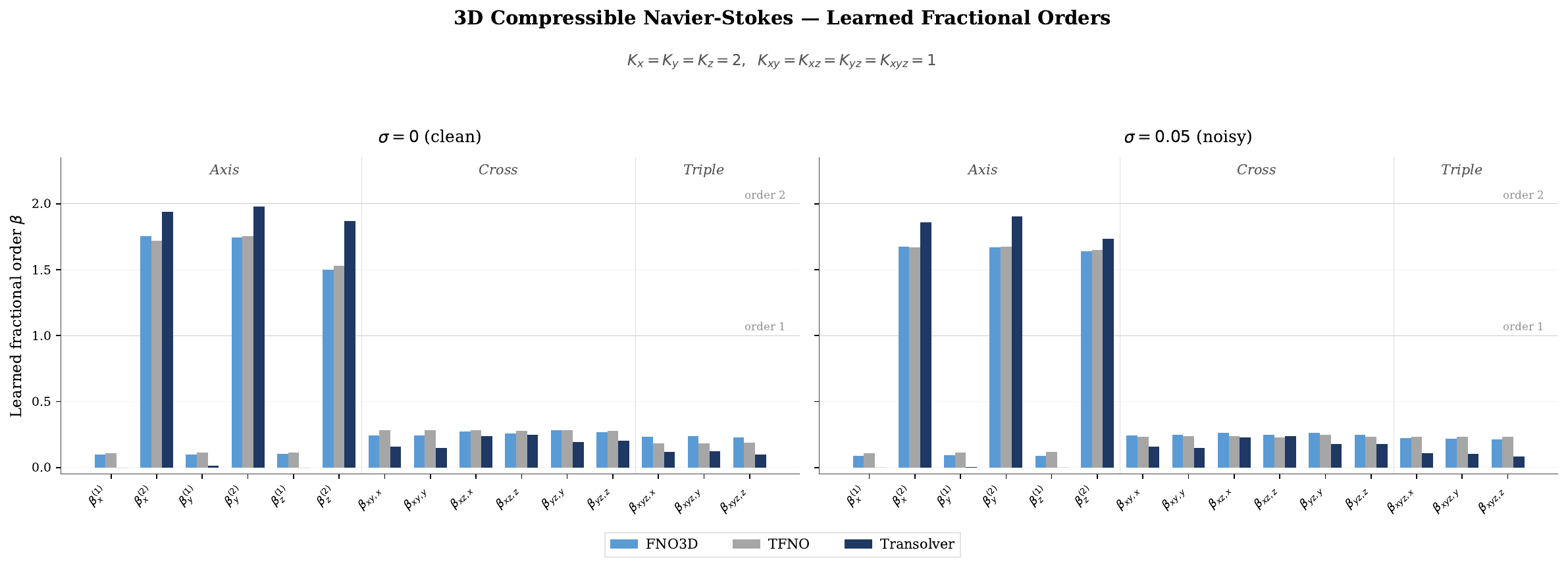}
\caption{Learned fractional orders on 3D compressible Navier-Stokes, for three backbones (FNO3D, TFNO, Transolver) under clean ($\sigma{=}0$, left) and noisy ($\sigma{=}0.05$, right) conditions. Each axis ($x, y, z$) retains a channel with order near $2$ (Laplacian, the dominant differential operator), while cross-derivative ($\beta_{xy}, \beta_{xz}, \beta_{yz}$) and triple-derivative ($\beta_{xyz}$) channels collapse to near-identity ($\beta \approx 0$). The network selectively amplifies the physically dominant differential structure without it being prescribed.}
\label{fig:cns_3d}
\end{figure}

Each axis retains a channel with order close to the Laplacian order $m{=}2$, while cross- and triple-derivative channels saturate near identity. The directional structure is consistent across three backbones and persists under noise, supporting the claim that learned $\beta$ values reflect physically meaningful differential structure rather than backbone-specific quirks.
\end{plainblock}

\subsection{Three-Dimensional Compressible Navier-Stokes: Quantitative Results}
\label{app:cns_3d_quant}

\begin{plainblock}
Complementing the learned-orders breakdown in \S\ref{app:cns_3d}, we report the quantitative test error on the 3D compressible Navier-Stokes benchmark ($64 \times 64 \times 64$, $T{=}10$, five-variable input). Three backbones (FNO3D, TFNO, Transolver) are evaluated under clean and noisy inputs; $\partial$-augmentation uses $K_x{=}K_y{=}K_z{=}2$ axis channels and $K_{xy}{=}K_{xz}{=}K_{yz}{=}K_{xyz}{=}1$ cross- and triple-derivative channels.

\begin{table}[h!]
\caption{3D Compressible Navier-Stokes: test relative $L^2$ error. $\partial$-augmentation improves every (backbone, noise) cell; gains widen under noise on Transolver.}
\label{tab:cns_3d_quant}
\centering
\resizebox{0.5\columnwidth}{!}{%
\begin{tabular}{@{}l cc cc cc@{}}
\toprule
& \multicolumn{2}{c}{\textbf{FNO3D}} & \multicolumn{2}{c}{\textbf{TFNO}} & \multicolumn{2}{c}{\textbf{Transolver}} \\
\cmidrule(lr){2-3}\cmidrule(lr){4-5}\cmidrule(lr){6-7}
& $\sigma{=}0$ & $\sigma{=}0.05$ & $\sigma{=}0$ & $\sigma{=}0.05$ & $\sigma{=}0$ & $\sigma{=}0.05$ \\
\midrule
Model            & 0.2101 & 0.2168 & 0.2079 & 0.2225 & 0.1990 & 0.2039 \\
$\partial$-Model & \textbf{0.1964} & \textbf{0.2115} & \textbf{0.1971} & \textbf{0.2033} & \textbf{0.1699} & \textbf{0.1695} \\
\rowcolor{gray!10} \textit{Improv.} & \impv{6.5\%} & \impv{2.4\%} & \impv{5.2\%} & \impv{8.6\%} & \impv{14.6\%} & \impv{16.9\%} \\
\bottomrule
\end{tabular}%
}
\end{table}

Transolver gains are largest both clean and noisy, consistent with its larger effective capacity and the bias-variance prediction that higher-capacity models benefit more from a regularization-aware input feature.
\end{plainblock}

\subsection{Conv-FNO Standalone Comparison}
\label{app:convfno}

\begin{plainblock}
We compare $\partial$-FNO against Conv-FNO \citep{liu2025enhancing} as direct alternatives: both preprocess the input, $\partial$-FNO with a GL derivative stencil and Conv-FNO with learnable convolution kernels. Both build on the same FNO backbone, allowing a clean head-to-head test of preprocessing strategies.

\begin{table}[h!]
\caption{Conv-FNO vs.\ $\partial$-FNO on three benchmarks. Relative $L^2$ error (\%); $\Delta$ is the relative improvement over the FNO baseline. $\partial$-FNO matches or outperforms Conv-FNO with \textbf{34--187$\times$ fewer} preprocessing parameters (see note below).}
\label{tab:convfno}
\centering
\resizebox{0.5\columnwidth}{!}{%
\begin{tabular}{@{}l c rrr cc@{}}
\toprule
& & & & & \multicolumn{2}{c}{\textit{Improv.\ over FNO}} \\
\cmidrule(lr){6-7}
\textbf{Dataset} & $\sigma$ & \textbf{FNO} & $\partial$\textbf{-FNO} & \textbf{Conv-FNO} & $\partial$\textbf{-FNO} & \textbf{Conv-FNO} \\
\midrule
Burgers ($\nu{=}0.1$) & 0    & 0.0832 & \textbf{0.0738} & 0.0791 & \impva{11.3\%} & \impva{4.9\%} \\
Burgers ($\nu{=}0.1$) & 0.02 & 0.1161 & \textbf{0.0899} & 0.1052 & \impva{22.6\%} & \impva{9.4\%} \\
\addlinespace
Darcy                 & 0    & 5.736  & \textbf{5.474}  & 5.687  & \impva{4.6\%}  & \impva{0.9\%} \\
Darcy                 & 0.02 & 5.721  & \textbf{5.487}  & 5.678  & \impva{4.1\%}  & \impva{0.8\%} \\
\addlinespace
NS ($\nu{=}10^{-5}$) & 0    & 3.506  & \textbf{3.456}  & 4.219  & \impva{1.5\%}  & $-20.3\%$ \\
NS ($\nu{=}10^{-5}$) & 0.02 & 3.614  & \textbf{3.575}  & 4.208  & \impva{1.1\%}  & $-16.4\%$ \\
\bottomrule
\end{tabular}%
}
\end{table}

\textit{Parameter counts.} On Burgers, $\partial$-FNO adds $4$ learnable preprocessing parameters (two orders $\beta_k$, two scales $s_k$, since the stencil taps are functions of $\beta_k$ rather than free parameters); Conv-FNO adds $\approx 136$ (16-wide kernels per channel). On Darcy/NS, $\partial$-FNO adds $11$ (six orders, five scales) vs Conv-FNO's $\approx 2056$. The constrained inductive bias matches or beats the learned kernel at 34--187$\times$ lower parameter overhead.
\end{plainblock}

\subsection{Ablations on $K$ and $K_{\mathrm{GL}}$}
\label{app:k_kgl}

\begin{plainblock}
Two architectural hyperparameters control the cost-quality tradeoff of $\partial$-augmentation: the number of learnable derivative channels $K$ and the length of the Grünwald-Letnikov stencil $K_{\mathrm{GL}}$. We ablate both on Burgers ($\nu{=}0.1$) under clean ($\sigma{=}0$) and noisy ($\sigma{=}0.02$) conditions.

\begin{figure}[h!]
\centering
\includegraphics[width=0.8\columnwidth]{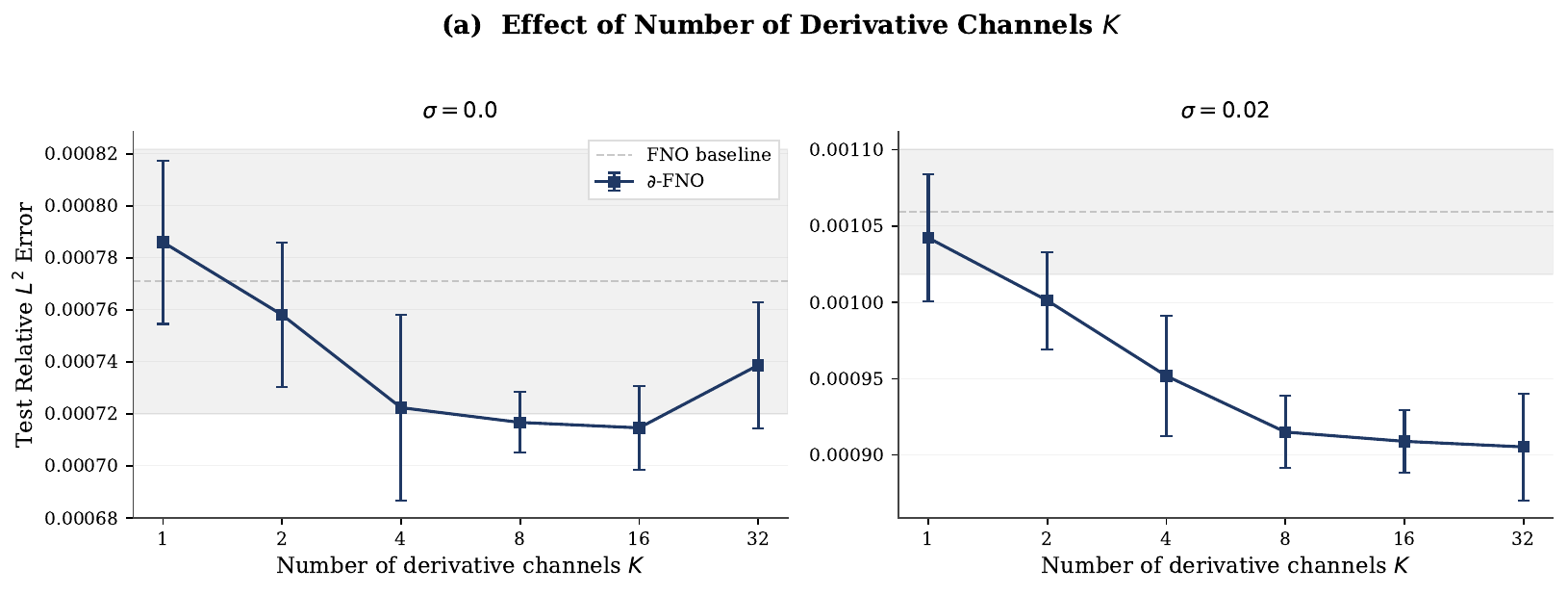}
\caption{Ablation on the number of derivative channels $K \in \{1, 2, 4, 8, 16, 32\}$. Increasing $K$ consistently improves performance with diminishing returns past $K{=}8$. The natural tradeoff is memory and compute cost. Under noise the curve flattens more slowly, indicating that additional channels help more when the regularization burden is larger.}
\label{fig:k_ablation}
\end{figure}
\begin{figure}[h!]
\centering
\includegraphics[width=0.8\columnwidth]{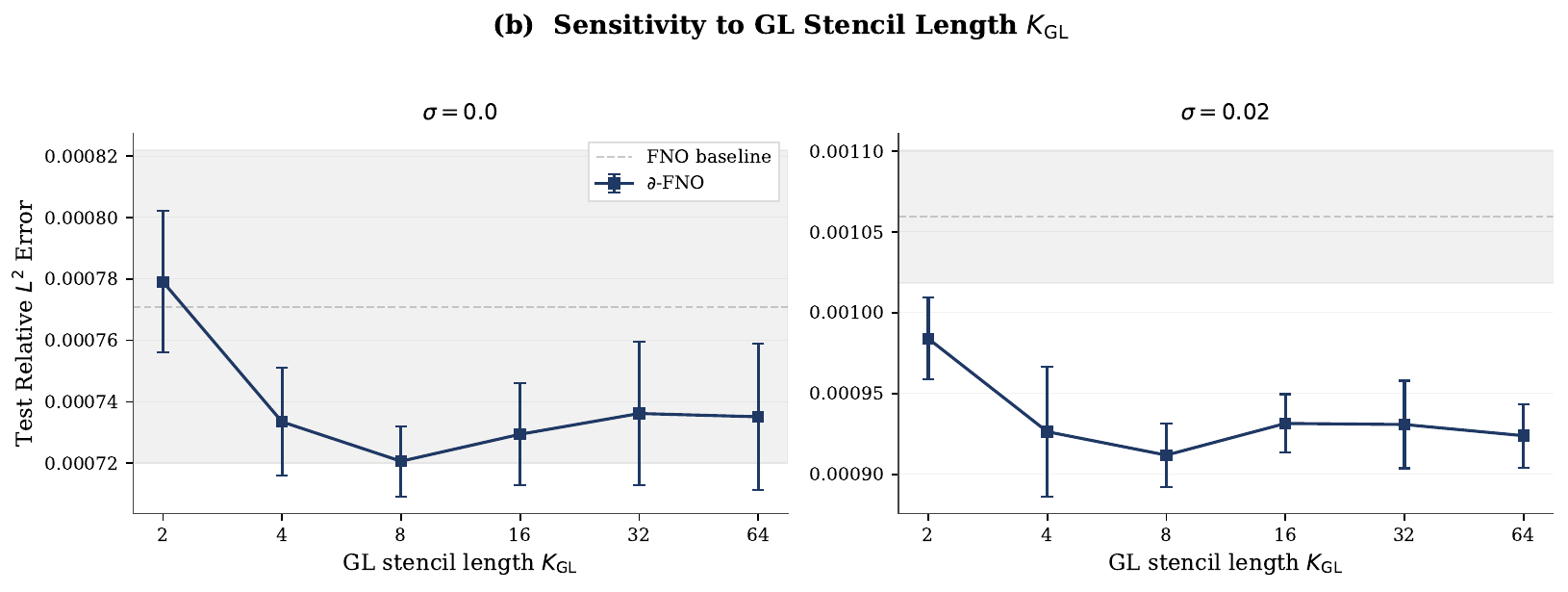}
\caption{Sensitivity to GL stencil length $K_{\mathrm{GL}} \in \{2, 4, 8, 16, 32, 64\}$. Performance stabilizes around $K_{\mathrm{GL}}{=}16$ under both clean and noisy conditions. At $K_{\mathrm{GL}}{=}2$, accuracy degrades because the stencil is too short to resolve non-integer orders. We use $K_{\mathrm{GL}}{=}16$ throughout the main paper as a balance between accuracy and compute.}
\label{fig:kgl_ablation}
\end{figure}

Both ablations support the default choices used in the main paper  ($K_{\mathrm{GL}}{=}16$ universally). Increasing $K$ further would improve accuracy at proportional memory cost; reducing $K_{\mathrm{GL}}$ below $16$ would save compute at the cost of resolution-dependent degradation of fractional orders.
\end{plainblock}

\subsection{Trajectory of Learned $\beta$ During Training}
\label{app:beta_trajectory}

\begin{plainblock}
Theorem~\ref{thm:optimal_beta} characterizes the population-optimal $\beta^*$, but in practice we recover $\beta$ via stochastic gradient descent. Figure~\ref{fig:beta_trajectory} traces all $K{=}8$ learnable orders across training epochs on Burgers ($\nu{=}0.1$), under five noise levels $\sigma \in \{0, 0.01, 0.02, 0.05, 0.1\}$.

\begin{figure}[h!]
\centering
\includegraphics[width=\textwidth]{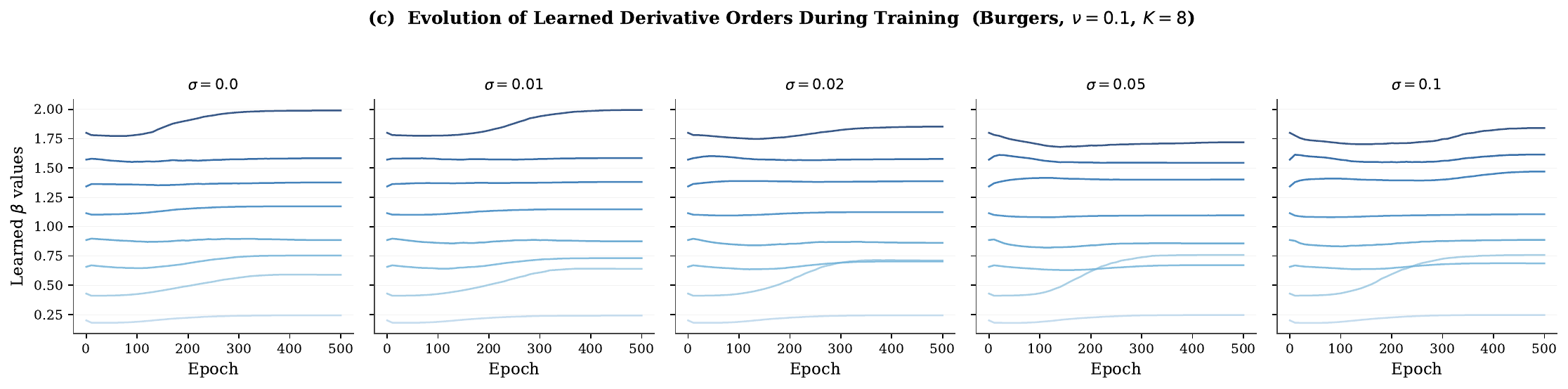}
\caption{Trajectory of learned $\beta$ values during training (5 noise levels). All $K{=}8$ orders stabilize within 200--400 epochs and converge to values consistent with Theorem~\ref{thm:optimal_beta}: higher $\beta^*$ under clean data, systematically lower $\beta^*$ as $\sigma$ grows. Initial values are $\beta_1{=}1.0$ and $\beta_2{=}2.0$; the optimizer redistributes channels across the spectral hierarchy as training proceeds.}
\label{fig:beta_trajectory}
\end{figure}

Three takeaways. First, all eight orders stabilize within 200--400 epochs, indicating optimization is reliable. Second, the final distribution is noise-dependent: at $\sigma{=}0$ orders cluster near $\beta \in [0.5, 2.0]$, while at $\sigma{=}0.1$ the cluster shifts toward $[0.3, 1.8]$, exactly the direction Theorem~\ref{thm:optimal_beta} predicts. Third, $\beta$ values redistribute across the channel budget rather than collapsing to redundant orders, validating the design choice of multiple learnable channels per direction.
\end{plainblock}

\subsection{Interaction Between Noise Level and Channel Count}
\label{app:noise_k_heatmap}

\begin{plainblock}
The K-ablation (Fig.~\ref{fig:k_ablation}) and the stress test (Table~\ref{tab:stress}) study the effects of channel count $K$ and noise level $\sigma$ separately. Figure~\ref{fig:noise_k_heatmap} shows their interaction: the relative improvement of $\partial$-FNO over FNO as a heatmap over the joint $(\sigma, K)$ plane on Burgers ($\nu{=}0.1$).
\begin{figure}[h!]
\centering
\includegraphics[width=0.55\columnwidth]{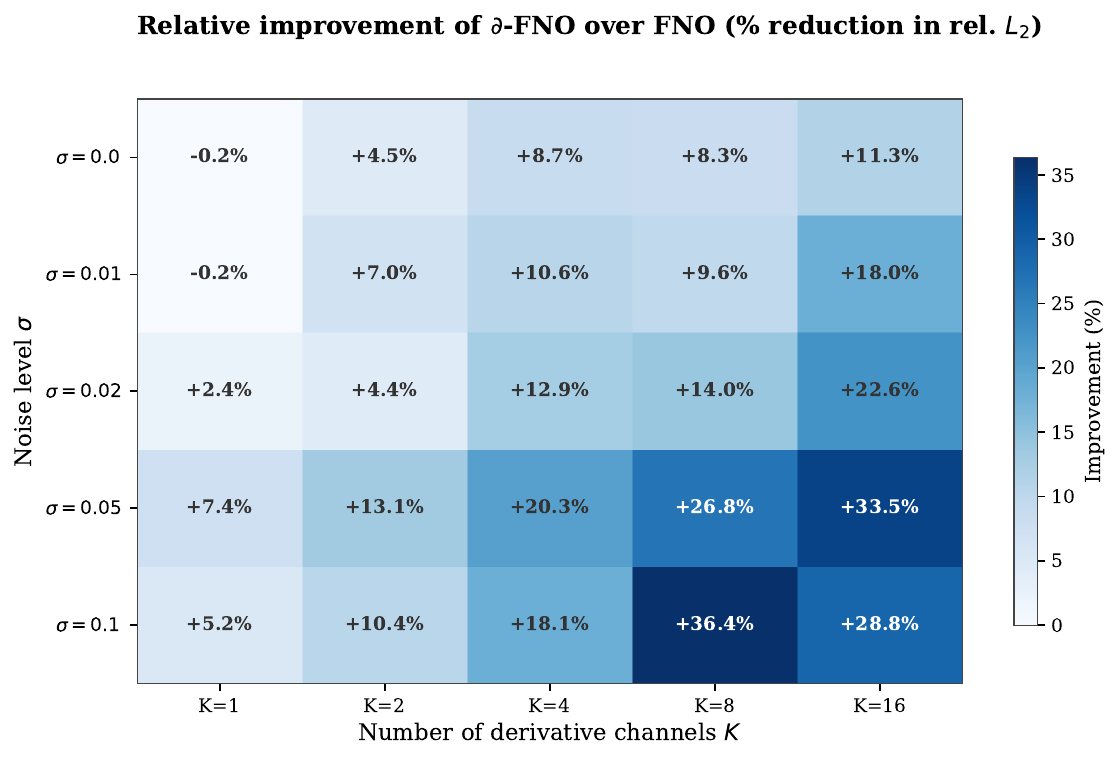}
\caption{Interaction between noise level $\sigma$ and number of derivative channels $K$. Cell values are the relative improvement of $\partial$-FNO over FNO (\% reduction in relative $L^2$) on Burgers ($\nu{=}0.1$). Gains broadly increase with both $\sigma$ and $K$ and concentrate in the high-noise, multi-channel corner, though the pattern is not monotone in either variable at these seed counts. The interaction is super-additive: the largest single gain ($+36.4\%$) sits at $(\sigma{=}0.1, K{=}8)$, in the high-noise regime near the diminishing-returns plateau of Figure~\ref{fig:k_ablation}, supporting the design choice of multiple channels precisely when data is noisy.}
\label{fig:noise_k_heatmap}
\end{figure}
 The super-additivity has a clean theoretical reading: each channel $\beta_k$ captures a different point on the bias-variance curve from Theorem~\ref{thm:optimal_beta}, so when noise broadens the optimal $\beta$ distribution the diverse-channel budget pays off more than the sum of its parts. The learnable spectral parameterization may extend to other frequency-selective domains such as signal detection under adverse SNR conditions \citep{mehouachi2021uav_gass, mehouachi2021uav_rsl}.

\end{plainblock}

\end{document}